\documentclass[11pt]{article}
\usepackage{bbm}
\usepackage{mathtools}

\usepackage{verbatim}
\usepackage{comment}
\usepackage[linesnumbered, ruled,vlined]{algorithm2e}

\usepackage{wrapfig}

\usepackage{multicol}
\usepackage{makecell}
\usepackage{multirow}
\usepackage{microtype}
\usepackage{graphicx}
\usepackage{booktabs} %

\usepackage{collectbox}

\usepackage{hyperref}

\usepackage{tikz}
\usetikzlibrary{intersections,shapes.arrows}

\usepackage{amsthm,amsmath,amssymb, amsfonts}
\usepackage{thmtools}

\newtheorem{theorem}{Theorem}[section]

\newtheorem{lemma}[theorem]{Lemma}

\newtheorem{definition}[theorem]{Definition}

\newtheorem{proposition}[theorem]{Proposition}
\newtheorem{assumption}{Assumption}[section]

\usepackage{comment}
\usepackage{xcolor}
\usepackage{mathrsfs}
\usepackage{enumitem}
\usepackage{bbm}
\usepackage{mdframed}
\usepackage{geometry}

\usepackage{pythonhighlight}

\newcommand{\Real}{\mathbb R}                        %
\def\1{\mathbf{1}}

\newcommand{\E}{{\mathbb E}}        

\usepackage{prettyref}

\newcommand{\savehyperref}[2]{\texorpdfstring{\hyperref[#1]{#2}}{#2}}
\newrefformat{eq}{\savehyperref{#1}{\textup{(\ref*{#1})}}}
\newrefformat{eqn}{\savehyperref{#1}{Equation~\ref*{#1}}}
\newrefformat{con}{\savehyperref{#1}{Conjecture~\ref*{#1}}}
\newrefformat{lem}{\savehyperref{#1}{Lemma~\ref*{#1}}}
\newrefformat{def}{\savehyperref{#1}{Definition~\ref*{#1}}}
\newrefformat{line}{\savehyperref{#1}{line~\ref*{#1}}}
\newrefformat{thm}{\savehyperref{#1}{Theorem~\ref*{#1}}}
\newrefformat{corr}{\savehyperref{#1}{Corollary~\ref*{#1}}}
\newrefformat{sec}{\savehyperref{#1}{Section~\ref*{#1}}}
\newrefformat{app}{\savehyperref{#1}{Appendix~\ref*{#1}}}
\newrefformat{ass}{\savehyperref{#1}{Assumption~\ref*{#1}}}
\newrefformat{ex}{\savehyperref{#1}{Example~\ref*{#1}}}
\newrefformat{fig}{\savehyperref{#1}{Figure~\ref*{#1}}}
\newrefformat{tab}{\savehyperref{#1}{Table~\ref*{#1}}}
\newrefformat{alg}{\savehyperref{#1}{Algorithm~\ref*{#1}}}
\newrefformat{rem}{\savehyperref{#1}{Remark~\ref*{#1}}}
\newrefformat{conj}{\savehyperref{#1}{Conjecture~\ref*{#1}}}
\newrefformat{prop}{\savehyperref{#1}{Proposition~\ref*{#1}}}
\newrefformat{proto}{\savehyperref{#1}{Protocol~\ref*{#1}}}
\newrefformat{prob}{\savehyperref{#1}{Problem~\ref*{#1}}}
\newrefformat{claim}{\savehyperref{#1}{Claim~\ref*{#1}}}

\usepackage{cdann_definitions}

\newmdtheoremenv{test}{Test}

\definecolor{DarkRed}{rgb}{0.75,0,0}
\definecolor{DarkGreen}{rgb}{0,0.5,0}
\definecolor{DarkPurple}{rgb}{0.5,0,0.5}
\definecolor{DarkBlue}{rgb}{0,0,0.7}

\hypersetup{colorlinks=true, plainpages = false, citecolor = DarkBlue, urlcolor = DarkBlue, filecolor = black, linkcolor =DarkBlue}

\usepackage[style = alphabetic,citestyle=alphabetic,maxbibnames=99,backend=bibtex,sorting=nyt,natbib=true,backref=true]{biblatex}
\DefineBibliographyStrings{english}{%
  backrefpage = {Cited on page},%
  backrefpages = {Cited on pages},%
}

\DeclareMathSizes{9}{8}{7}{5}
\title{\textbf{Model Selection in Contextual Stochastic Bandit Problems}}
\author{Aldo Pacchiano\footnote{Equal contribution.}\phantom{$^\ast$} \\ University of California Berkeley \\ \texttt{pacchiano@berkeley.edu} \and My Phan$^\ast$ \\ University of Massachusetts \\ \texttt{myphan@cs.umass.edu}  \and Yasin Abbasi-Yadkori \\ DeepMind \\ \texttt{yadkori@google.com} \and Anup Rao \\ Adobe \\ \texttt{anuprao@adobe.com} \and Julian Zimmert \\ Google Research \\ \texttt{zimmert@google.com} \and Tor Lattimore \\ DeepMind \\ \texttt{lattimore@google.com} \and Csaba Szepesv\'{a}ri \\ DeepMind \& University of Alberta \\ \texttt{szepi@google.com}}
\date{\today}
\begin{document}
\maketitle
\begin{abstract}
We study bandit model selection in stochastic environments. Our approach relies on a meta-algorithm that selects between candidate base algorithms. We develop a meta-algorithm-base algorithm abstraction that can work with general classes of base algorithms and different type of adversarial meta-algorithms. Our methods rely on a novel and generic smoothing transformation for bandit algorithms that permits us to obtain optimal $O(\sqrt{T})$ model selection guarantees for stochastic contextual bandit problems as long as the optimal base algorithm satisfies a high probability regret guarantee.  We show through a lower bound that even when one of the base algorithms has $O(\log T)$ regret, in general it is impossible to get better than $\Omega(\sqrt{T})$ regret in model selection, even asymptotically. Using our techniques, we address model selection in a variety of problems such as misspecified linear contextual bandits \citep{lattimore2019learning}, linear bandit with unknown dimension \citep{foster2019model} and reinforcement learning with unknown feature maps. Our algorithm requires the knowledge of the optimal base regret to adjust the meta-algorithm learning rate. We show that without such prior knowledge any meta-algorithm can suffer a regret larger than the optimal base regret. 
\end{abstract}
\section{Introduction}\label{section::introduction}

Bandit algorithms have been applied in a variety of decision making and personalization problems in industry. There are many specialized algorithms each designed to perform well in specific environments. For example, algorithms are designed to exploit low variance~\citep{AMS-2009}, extra context information, linear reward structure~\citep{Dani-Hayes-Kakade-2008, li+10, abbasi2011improved}, sparsity~\citep{APS-2012,CM-2012}, etc. The exact properties of the current environment however might not be known in advance, and we might not know which algorithm is going to perform best.

Model selection in contextual bandits aims to solve this problem. More formally, the learner is tasked to solve a bandit problem for which the appropriate bandit algorithm to use is not known in advance. Despite this limitation, the learner does have access to $M$ different algorithms $\{\mathcal{B}_i\}_{i=1}^M$, one of which $\mathcal{B}_{i_\star}$ is promised to be adequate for the problem the learner wishes to solve. We use regret to measure the learner's performance\footnote{We will define regret more formally in the following section.}. The problem's objective is to design algorithms to minimize regret. 

The algorithms we develop in this work follow the template of~\citep{DBLP:conf/colt/AgarwalLNS17} where a `meta-algorithm' algorithm is placed on top of a couple of `base' algorithms (in this case $\{\mathcal{B}_i\}_{i=1}^M$). At the beginning of each round the meta-algorithm selects which base algorithm to `listen to' during that time-step effectively treating the base algorithms as arms to be pulled by the meta-algorithm. The difficulty in using existing algorithms such as UCB or EXP3~\citep{DBLP:journals/corr/abs-1204-5721} as a meta-algorithm lies in the non-stationary nature of the rewards collected by a learning base algorithm. The meta-algorithm needs to be sufficiently smart to recognize when a base algorithm is simply performing poorly because it still in the early stages of learning from the case where poor performance is the result of model misspecification.

\paragraph*{Adapted and misspecified algorithms} We say that an algorithm is \emph{adapted} to the environment at hand if it satisfies a valid regret guarantee. Let's illustrate this with an example in the setting of linear bandits with finitely many arms. In this problem the learner has access to $K$ arms. Each arm $i \in [K]$ is associated with a feature vector $z_i \in \mathbb{R}^d$, and the reward of arm $i \in [K]$ follows a linear model of the form $r_i  = \langle z_i, \theta_\star \rangle + \xi_i$ where $\xi_i$ is conditionally zero mean and $\theta_\star$ is an unknown parameter. An algorithm such as LinUCB~\citep{pmlr-v15-chu11a} achieves a regret guarantee of order $\widetilde{\mathcal{O}}( \sqrt{d \log^3(K)T}  )$ where $\widetilde{\mathcal{O}}$ hides logarithmic factors in $T$. In contrast, the UCB algorithm~\citep{auer2002finite} yields a regret guarantee of order $\widetilde{\mathcal{O}}( \sqrt{K T}  )$. In this case, both algorithms are well adapted to the problem of linear bandits with finitely many actions, but LinUCB's regret guarantee may be substantially smaller than UCB's regret upper bound if $d$ is much smaller than $K$. If an algorithm is not well adapted, we say it is \emph{misspecified}. For the sake of exposition let's assume we are in a similar setting as above, where the learner has access to $K$ arms each of which is associated with a feature vector $z_i \in \mathbb{R}^d$. Instead of assuming a linear model as before, let's instead assume that $r_i =\left(\langle z_i, \theta_\star \rangle\right)^2 + \xi_i $ is quadratic. In this case, there is no reason to believe LinUCB can yield a valid regret guarantee since the underlying linearity assumption of LinUCB is violated. We say that in this case LinUCB is misspecified. Consider an instance of LinUCB that instead uses matrix features of the form $z_i z_i^\top$. In this case the quadratic reward is again a linear function of the feature vectors since $ \left(\langle z_i, \theta_\star \rangle\right)^2 =  \langle z_i z_i^\top,  \theta_\star \theta_\star^\top \rangle $. Thus this version of LinUCB with quadratic features is adapted.

We will assume that all algorithms $\mathcal{B}_i$ for $i \in [M]$ are associated with a putative regret guarantee $U_i(t, \delta )$ that is \emph{known} by the learner and holding with probability $1-\delta$ for all $t \in [\mathbb{N}]$ \emph{if algorithm $i$ is adapted} to the environment at hand.  If the learner knew the identity of the best adapted algorithm $i_\star$, it would be able to incur regret of order $U_{i_\star}(T, \delta)$ by playing $\mathcal{B}_{i_\star}$ for $T$ time-steps. The learner's objective in the model selection problem is to design a procedure that would allow a learner to incur in regret that is competitive with the regret upper bound $U_{i_\star}(t, \delta )$ of the best adapted algorithm among those in $\{\mathcal{B}_i\}_{i=1}^M$, so that the regret incurred by the learner up to time $T$ scales as a function of $T$, the parameters defining $\mathcal{B}_{i_\star}$ (and therefore $U_{i_\star}(t, \delta)$) and possibly $M$. From now on we will refer to each of the $M$ algorithms in $\mathcal{B}_{i_\star}$ as a \emph{base algorithm}. We will alert the reader if we have a specific set of $M$ algorithms in mind. In any other case, when we talk about the set of base algorithms we simply mean a set of $M$ algorithms the learner is model selecting from.

The authors of \citep{Maillard-Munos-2011} were perhaps the first to address the bandit model-selection problem, with a variant of an EXP4 meta-algorithm that works with UCB or EXP3 base algorithms. These results are improved by \citep{DBLP:conf/colt/AgarwalLNS17} via the CORRAL algorithm. The CORRAL algorithm follows the meta-algorithm-base template that we discussed at the beginning of this section. It makes use of a CORRAL meta-algorithm based on a Log-Barrier Online Mirror Descent algorithm controlling which of the base algorithms to play at any given round. Let $p_t$ be the $M-$dimensional probability distribution over the $M$ base algorithms given by the CORRAL meta-algorithm. The learner will sample an algorithm index $j_t \in [M]$ with $j_t \sim p_t$. and play the action prescribed by $\mathcal{B}_{j_t}$ to collect a reward $r_t$. All algorithms $\{ \mathcal{B}_i \}_{i=1}^M$ are then updated using an importance weighted version of $r_t$ regardless of whether they were selected by the meta-algorithm or not. 

Unfortunately, this means that in order to use a base algorithm in CORRAL, it needs to be compatible with this importance weighting modification of the rewards. For example, to use UCB as a base, we would need to manually re-derive UCB's confidence intervals and modify its regret analysis to be compatible with importance weighted feedback. The authors show that a base algorithm can be safely combined with the CORRAL meta-algorithm to yield model selection guarantees provided it satisfies a stability condition (see Definition 3 in \citep{DBLP:conf/colt/AgarwalLNS17}). Verifying that an algorithm satisfies such stability condition is a cumbersome process that requires a detailed analysis of the algorithm's internal workings. In this work we instead focus on devising a black-box procedure that can solve the model selection problem for a general class of stochastic contextual bandit algorithms. This work introduced the first black-box method for model selection in stochastic contextual bandits, and it has been followed by many others that have expanded and refined these results; most notably~\citep{pacchiano2020regret,cutkosky2021dynamic,lee2020online,abbasi2020regret,pacchiano2022best,arora2021corralling}.

 \paragraph*{Contributions. } We focus on the problem of bandit model-selection in stochastic environments. Our contributions are as follows:
 \begin{itemize}[leftmargin=*]
 \item  \textbf{A new model selection algorithm}. We introduce Stochastic CORRAL, a two step per round algorithm and an accompanying base ``smoothing" wrapper that can be shown to satisfy model selection guarantees when combined with any set of $M$ stochastic contextual bandit algorithms such that at least one of them is adapted and satisfies a high probability regret guarantee. We also show model selection regret guarantees for Stochastic CORRAL with two distinct adversarial meta-algorithms, CORRAL \citep{DBLP:conf/colt/AgarwalLNS17} and EXP3.P \citep{DBLP:journals/corr/abs-1204-5721}. Our approach is more general than that of the original CORRAL algorithm~\citep{DBLP:conf/colt/AgarwalLNS17} because instead of requiring each base algorithm to be individually modified to satisfy a certain stability condition, our version of the CORRAL algorithm provides the algorithm designer with a generic black-box wrapper that allows to do model selection over any set of $M$ base algorithm with high probability regret guarantees. Stochastic CORRAL has another important difference with respect to the original CORRAL algorithm: instead of importance weighted feedback, the unadulterated reward $r_t$ is sent to algorithm $\mathcal{B}_{j_t}$, and only this algorithm is allowed to update its internal state at round $t$. The main consequence of these properties of Stochastic CORRAL is that our model selection strategy can be used with almost any base algorithm developed for stochastic environments. When the learner has knowledge of the function $U_{i_\star}(t, \delta)$ but not of $i_\star$ (for example when all the putative upper bounds $U_i(t, \delta)$ are the same), using the CORRAL meta-algorithm achieves optimal regret guarantees. When the optimal target regret is unknown sometimes using a an EXP3.P meta-algorithm can achieve better performance. 
 \item \textbf{A versatile algorithm.} We demonstrate the generality and effectiveness of our method by showing how it seamlessly improves existing results or addresses open questions in a variety of problems. We show applications in adapting to the misspecification level in contextual linear bandits \citep{lattimore2019learning}, adapting to the unknown dimension in nested linear bandit classes \citep{foster2019model}, tuning the data-dependent exploration rate of bandit algorithms, and choosing feature maps in reinforcement learning. Moreover, our meta-algorithm can simultaneously perform different types of model selection. For example, we show how to choose both the unknown dimension and the unknown mis-specification error at the same time. This is in contrast to algorithms that specialize in a specific type of model selection such as detecting the unknown dimension \citep{foster2019model}. 
 \item \textbf{Lower bounds.} In the stochastic domain, an important question is whether a model selection procedure can inherit the $O(\log T)$ regret of a fast stochastic base algorithm. We show a lower bound for the model selection problem that scales as $\Omega(\sqrt{T})$, which implies that our result is minimax optimal. Recall the CORRAL meta-algorithm achieves an oracle optimal model selection regret guarantee provided the algorithm has access to $U_{i_\star}(t, \delta)$. This begs the question of whether this is an unavoidable statistical limitation of model selection or just a property of the CORRAL meta-algorithm. We show this condition is unavoidable in general: there are problems where the regret of the best base algorithm scales as $O(T^x)$ for an unknown $x$, and the regret of any meta-algorithm that is unaware of the value of $x$ scales as $\Omega(T^y)$ for $y > x$. %
 \end{itemize}

\section{Problem Statement}
We use the notation $\delta_a$ to write the delta distribution at $a$. For an integer $n$, we use $[n]$ to denote the set $\{1,2,\dots,n\}$. 
We consider the following formulation of contextual stochastic bandits. At the beginning of each time-step $t$, the learner observes a context $\mathcal{A}_t$ that corresponds to a subset of an `action-set' $\mathbb{A}$. After this the learner will select an action $a_t \in \mathcal{A}_t$ and then collect a reward $r_t = f(\mathcal{A}_t, a_t) + \xi_t$, a noisy quantity that will depend on the context $\mathcal{A}_t$, and the learner's action $a_t$, a reward function $f$ and a $1-$subGaussian conditionally zero mean random noise random variable $\xi_t$. In this work we will restrict ourselves to the case where contexts sets $\mathcal{A}_t$ are all subsets of a context generating set $\mathbb{A}$. This is in fact a very general scenario that captures all types of contextual bandit problems ranging from the case of changing linear contexts with linear rewards, to more general contexts and reward sets studied in works such as~\citep{foster2020beyond}. For simplicity we will assume the contexts $\mathcal{A}_t$ are made of the subset of available actions to the learner at time $t$. Our formulation allows for the action set to vary in size from round to round and even to be infinite. For example, the finite linear bandit setting (where $\mathcal{A}_t = \mathbb{A} \subset \mathbb{R}^d$ for all $t$) fits in this setting. Similarly it is easy to see the linear contextual bandit problem with i.i.d. contexts and $K$ actions can also be written as an instance of our formulation. In the linear contextual bandit problem with $K_t$ actions (where $K_t$ may be infinite) the learner is presented at time $t$ with $K_t$ action-vectors $\mathcal{A} = \{ a_i \}_{i=1}^{K_t}$ with $a_i \in \mathbb{R}^d$ and the (random) return $r_a$ of any action $a \in \mathcal{A}$ satisfies $r_a = \langle a , \theta_\star \rangle + \xi$. The $K-$action setting where contexts $x_t \in \mathcal{X}$ (an abstract context set) and the learner has access to $K$ actions per round can also be formulated in this way by defining $\mathbb{A} = \mathcal{X} \times [K]$ and defining $\mathcal{D}_S$ to be a distribution over subsets of the form $\{(x, i)\}_{i=1}^K$ with $x \in \mathcal{X}$. 

In this work we focus on the setting of stochastic i.i.d. contexts.  Let $S$ be the set of all subsets of $\mathbb{A}$ and let $\mathcal{D}_S$ be a distribution over $S$. We assume all contexts $\mathcal{A}_t \stackrel{i.i.d.}{\sim} D_{S}$ and that $f: S \times \mathbb{A} \rightarrow \mathbb{R}$.  For an arbitrary subset $\mathcal{A} \subset \mathbb{A}$ we denote the space of distributions over $\mathcal{A}$ as $\Delta_{\mathcal{A}}$. A policy $\pi$ is a mapping such that for any subset $\mathcal{A} \subset \mathbb{A}$ in the support of $\mathcal{D}_S$ outputs an element of $\Delta_{\mathcal{A}}$. Let's denote by $\Pi$ as the space of all policies with domain in $\mathrm{Support}(\mathcal{D}_S)$. We abuse notation and denote $f(\mathcal{A}, \pi) = \mathbb{E}_{ a \sim \pi(\mathcal{X})} \left[ f(\mathcal{A} , a )\right]$ . Notice that in this case $f(\mathcal{A}, a) = f(\mathcal{A}, \delta_a)$ for all $a \in \mathcal{A}$. We will generally omit the $\mathcal{A}\sim \mathcal{D}_S$ dependence from our expectation notation in the future. 

In a contextual bandit problem the learner chooses policy $\pi_t$ at time $t$, which takes context set  $\mathcal{A}_t \in S$ as an input and outputs a distribution over $\mathcal{A}_t$. The learner then selects an action $a_t  \sim \pi_t(\mathcal{A}_t)$ and receives a reward $r_t$ such that ${r_t} = f(\mathcal{A}_t, \delta_{a_t}) + \xi_t$.

We are interested in designing an algorithm with small regret, defined as 
\begin{equation}\label{equation::regret_definition}
R(T) = \mathop{\max}_{\pi \in \Pi} \EE{\sum_{t=1}^{T} f(\mathcal{A}_t, \pi) - \sum_{t=1}^{T} f(\mathcal{A}_t,\pi_t)}.
\end{equation}

If for example $U_{i}({T}, \delta) = c d_{i} \sqrt{ T \log(1/\delta)}$ for all $i \in [M]$ we would like our algorithm to satisfy a regret guarantee of the form $R(T) \leq \mathcal{O}( M^\alpha d_{i_\star}^\beta \sqrt{T\log(1/\delta)} )$ for some $\alpha \geq 0, \beta \geq 1 $ and where $i_\star$ is the index of the best performing adapted base algorithm $\mathcal{B}_{i_\star}$. Crucially, we want to avoid this guarantee to depend on other $d_i > d_{i_\star}$ (if any). From now on we will refer to the policy maximizing the right hand side of the equation above as $\pi^*$. For simplicity we will also make the following assumption regarding the range of $f$,

\begin{assumption}[Bounded Expected Rewards]\label{assumption:unit_bounded_reward}
The absolute value of $f$ is bounded by $1$,
\begin{equation*}
    \max_{\mathcal{A}' , \pi } \left|  f(\mathcal{A}', \pi)  \right| \leq 1
\end{equation*}
\end{assumption}

Throughout this work we assume the base algorithms we want to model select from satisfy a high probability regret bound whenever they are well adapted to their environment. We make this more precise in definition~\ref{definition::boundedness},

\begin{definition}[$(U, \delta, T)-$Boundedness]\label{definition::boundedness}
Let $U : \mathbb{R}\times [0,1] \rightarrow \mathbb{R}^+$. We say an adapted algorithm $\mathcal{B}$ is $(U, \delta, T)-$bounded if with probability at least $1-\delta$ and for all rounds $t\in [1,T]$, its cumulative pseudo-regret  is bounded above by $U(t, \delta)$: 
  $  \sum_{l = 1}^t f(\mathcal{A}_l, \pi^* ) - f(\mathcal{A}_l, \pi_l)   \leq U(t, \delta).$
\end{definition}

We assume that for all $i \in [M]$ the base algorithm $\mathcal{B}_i$ is $(U_i, \delta, T)$-bounded for a function $U_i$ known to the learner\footnote{Recall that in this case the upper bound on the algorithm's regret is satisfied only when $\mathcal{B}_i$ is well adapted to the environment.}. For example in the Multi Armed Bandit Problem with $K$ arms the UCB algorithm is $( c\sqrt{KT\log(T/\delta) }, \delta, T)-$bounded for some universal constant $c > 0$. 

\subsubsection*{Original CORRAL}
We start by reproducing the pesudo-code of CORRAL~\cite{DBLP:conf/colt/AgarwalLNS17} (see Algorithm~\ref{Alg:meta-algorithm_full}) as it will prove helpful in our discussion of our main algorithm: Stochastic CORRAL. As we have explained in the previous section we assume there are $M$ candidate \emph{base} algorithms and a meta-algorithm which we denote as $\mathcal{M}$. At time-step $t$ the CORRAL meta-algorithm $\mathcal{M}$ selects one of the base algorithms in $\{\mathcal{B}_i\}_{i=1}^M$ according to a distribution $p_t  \in \Delta_M$ by sampling an index $j_t \sim p_t$. The learner plays action $a_t \sim \pi_{t, j_t}(\mathcal{A}_t)$ and receives reward $r_t = f(\mathcal{A}_t, \delta_{a_t}) + \xi_t$. An importance weighted version of $r_t$ is sent out to all base algorithms, after which all of them update their internal state.

\begin{algorithm}[h]
\textbf{Input:} Base Algorithms $\{\mathcal{B}_j\}_{j=1}^M$, learning rate $\eta$.  \\
Initialize: $\gamma = 1/T, \beta = e^{\frac{1}{\ln T}}, \eta_{1,j} = \eta, \rho^j_{1} = 2M, \underline{p}^j_{1} = \frac{1}{ \rho^j_{1}},  {p}^j_1= 1/M$ for all $j \in [M]$. \\
Initialize all base algorithms. 

\For{$t = 1, \cdots, T$}{
Receive context $\mathcal{A}_t \sim \mathcal{D}_S$.\\
Receive policy $\pi_{t,j}$ from $\mathcal{B}_{j}$ for all $j \in [M]$. \\
Sample $j_t \sim p_t$. \\
Play action $a_t \sim \pi_{t,{j_t}}(\mathcal{A}_t)$. \\
Receive feedback $r_t = f(\mathcal{A}_t, \delta_{a_t}) + \xi_t$. \\ 
Send feedback $\frac{r_{t}}{\overline{p}_{t,j_t}}\mathbf{1}\{j = j_t\}$ to $\mathcal{B}_{j}$ for all $j \in [M]$.\\
Update $p_{t}$, $\eta_t$, $\underline{p}_t$ and $\rho_t$ to $p_{t+1}$, $\eta_{t+1}$, $\underline{p}_{t+1}$ and $\rho_{t+1}$ via $\mathsf{CORRAL-Update}$ \\
}
\caption{Original CORRAL}
\label{Alg:meta-algorithm_full}
\end{algorithm}
\begin{algorithm}[h]
\textbf{Input:} learning rate vector $\eta_t$, previous distribution $p_t$ and current loss $\ell_t$ \\
\textbf{Output:} updated distribution $p_{t+1}$ \\
Find $\lambda \in [\min_{j} \ell_{t,j} , \max_{j} \ell_{t,j} ]$ such that $\sum_{j=1}^M \frac{ 1}{ \frac{1}{p^i_{t}} + \eta_{t,j}(\ell_{t,j} - \lambda)   }= 1$\\
Return ${p}_{t+1}$ such that $\frac{1}{p^j_{t+1}} = \frac{1}{p^j_{t}} + \eta_{t,j}(\ell_{t,j} - \lambda)$\\
 \caption{Log-Barrier-OMD($p_t, \ell_t, \eta_t$) }
\label{Alg:log_barrier}
\end{algorithm}

\begin{algorithm}[H]
\textbf{Input:} learning rate vector $\eta_t$, distribution $p_t$, lower bound $\underline{p}_t$ and current loss $r_t$ \\
\textbf{Output:} updated distribution $p_{t+1}$, learning rate $\eta_{t+1}$ and loss range $\rho_{t+1}$ \\
Update $p_{t+1} = \text{Log-Barrier-OMD}(p_t, \frac{r_{t}}{{p}_{t,j_t}}\mathbf{e}_{j_t}, \eta_t)$. \\
Set ${p}_{t+1} = (1-\gamma)p_{t+1} + \gamma \frac{1}{M}$. 

\For{$j =1, \cdots, M$}{
\If{$ \underline{p}^j_{t} > {{p}^j_{t+1}} $}{
Set $\underline{p}^j_{t+1} = \frac{{p}^j_{t+1}}{2}, \eta_{t+1, j} = \beta\eta_{t,i}$, \\
}
\Else{
 Set $\underline{p}^j_{t+1}=\underline{p}^j_{t}, \eta_{t+1,j} = \eta_{t,i}$. \\
}
Set $\rho^j_{t+1} = \frac{1}{\underline{p}^j_{t+1}}$.\\
}
Return $p_{t+1}$, $\eta_{t+1}$, $\underline{p}_{t+1}$ and $\rho^j_{t+1}$. 
\caption{$\mathsf{CORRAL-Update}$}\label{Alg:corral_update}
\end{algorithm}

\section{The Stochastic CORRAL Algorithm}\label{sec:the_stochastic_corral_algorithm}

In order to better describe the feedback structure of Stochastic CORRAL we abstract the meta-algorithm-base interaction template discussed in Section~\ref{section::introduction} into Algorithms~\ref{Alg:meta-algorithm} and~\ref{Alg:base}. As we have mentioned before, one crucial difference between Stochastic CORRAL and CORRAL is that in Stochastic CORRAL only the state of the base algorithm whose action was selected is modified. In contrast in the CORRAL algorithm all the base algorithms' states are updated during every step. 

To make this description more precise we introduce some notation. Each base algorithm $\mathcal{B}_j$ maintains a counter $s_{t, j}$ that keeps track of the number of times it has been updated up to time $t$. We will say algorithm $\mathcal{B}_j$ is in `state' $s_{t,j}$ at time $t$. For any base algorithm $\mathcal{B}_j$, $\pi_{s,j}$ is the policy $\mathcal{B}_j$ uses at state $s$.  If $t_1< t_2$ are two consecutive times when base $j$ is chosen by the meta-algorithm, then base $j$ proposed policy $\pi_{s_{t_1,j},j}$ at time $t_1$ and policy $\pi_{s_{t_2,j},j}$ during all times $t_1 +1, \dots, t_2$ where $s_{t_2,j} = s_{t_1,j} + 1$. 

\begin{algorithm}[h]
 \textbf{Input:} Base Algorithms $\{\mathcal{B}_j\}_{j=1}^M$
\For{ $t = 1, \cdots, T$ }{ 
 Sample $j_t \sim p_t$.\\
 Play $j_t$.\\ 
Receive feedback $r_t = r_{t,j_t}$ from playing the action prescribed by $\mathcal{B}_{j_t}$ \\
 Update meta-algorithm using $r_t$
 }
 \caption{Meta-Algorithm}
\label{Alg:meta-algorithm}
\end{algorithm}

\begin{algorithm}[h]
Initialize state counter $s = 1$
\For{ $t = 1, \cdots, T$ }{ 
Receive action set $\mathcal{A}_t \sim \mathcal{D}_S$\\
Choose action $a_{t,j} \sim \pi_{s,j}(\mathcal{A}_t)$ \\
\If { selected by meta-algorithm (i.e. $j_t = j$) }{
Play action $a_{t,j}$ \\
Receive feedback $r_{t,j} = f(\mathcal{A}_t, \delta_{a_{t,j}}) + \xi_t$ \\
Send $r_{t,j}$ to the meta-algorithm \\
Compute $\pi_{s+1, j}$ using $r _{t,j}$\\
$s \leftarrow s+1$

}
 }
 \caption{Base Algorithm $\mathcal{B}_j$}
\label{Alg:base}
\end{algorithm}

\paragraph*{Regret Decomposition.} Let's introduce the regret decomposition we will make use of to prove our regret guarantees. This is a similar decomposition as the one appearing in the proofs of Theorem 4,5 and  7 of~\citep{DBLP:conf/colt/AgarwalLNS17}. We split the regret $R(T)$ of Equation~\ref{equation::regret_definition} into two terms ($\mathrm{I}$ and $\mathrm{II}$) by adding and subtracting terms $\{  f(\mathcal{A}_t, \pi_{s_{t,{i_\star}},i_\star})\}_{t=1}^T$ :
\begin{align}
\label{eq::regret_decomp}
R(T) &= \EE{\sum_{t=1}^T f(\mathcal{A}_t, \pi^*) - f(\mathcal{A}_t, \pi_t) } \notag \\ &=\mathbb{E}\underbrace{ \left[\sum_{t=1}^T f(\mathcal{A}_t, \pi_{s_{t,i_\star},i_\star}) - f(\mathcal{A}_t, \pi_t)\right]}_{\mathrm{I}}+
 \mathbb{E}\underbrace{\left[\sum_{t=1}^T   f(\mathcal{A}_t, \pi^*) - f(\mathcal{A}_t, \pi_{s_{t,i_\star},i_\star})\right]}_{\mathrm{II}} 
\end{align}
Term $\mathrm{I}$ is the regret of the meta-algorithm with respect to the optimal base $\mathcal{B}_{i_\star}$, and term $\mathrm{II}$ is the regret of the optimal base with respect to the optimal policy $\pi^*$. %
Analysis of term $\mathrm{I}$ is largely based on the adversarial regret guarantees of the Log-Barrier-OMD in CORRAL and of the EXP3.P algorithm.

In order to bound term $\mathrm{II}$ we will have to modify the feedback structure of Algorithms~\ref{Alg:meta-algorithm} and~\ref{Alg:base}. In Algorithm~\ref{Alg:smoothing} from Section~\ref{sec::base_smoothing} we introduce a smoothing procedure that allows any $(U, \delta, T)$-bounded algorithm to be transformed into a `smoothed' version of itself such that its conditional expected instantaneous regret is bounded with high probability during every even step. We name this procedure `smoothing' because it is based on playing uniformly from the set of previously played policies during the smoothed algorithm's odd steps. We provide more details in Section~\ref{sec::base_smoothing}. For now, the main property we are to use from this discussion is that by smoothing a $(U, \delta, T)$-bounded algorithm it is possible to ensure the conditional expected instantaneous regret of the smoothed algorithm is bounded above by $\frac{U(\ell, \delta)}{\ell}$ during the $\ell-$th even step. The function $\frac{U(\ell, \delta)}{\ell}$ can be shown to be decreasing (as a function of $\ell$) when $U(\ell, \delta)$ is concave in $\ell$. In Stochastic CORRAL the smoothing of base algorithms takes the place of the stability condition required by the CORRAL algorithm in~\citep{DBLP:conf/colt/AgarwalLNS17}.

Let's sketch some intuition behind why this decreasing instantaneous regret condition can help us bound term $\mathrm{II}$. For all $i \in [M]$ let $\{p^i_{1}, \dots, p^i_T\}$ be the (random) probabilities used by the Stochastic CORRAL meta-algorithm $\mathcal{M}$ (an adversarial meta-algorithm) to chose base $i$ during round $t$ and let $\underline{p}_i = {\min_t p_t^i}$. Let's focus on the optimal algorithm $i_\star$ and assume $U_{\star}(t, \delta)$ is convex in $t$. Since $\frac{U_\star(t, \delta)}{t}$ is decreasing, term $\mathrm{II}$ is the largest when base $i_\star$ is selected the least often. For the sake of the argument let's assume that $p^{i_\star}_t =\underline{p}_{i_\star}~\forall t$. In this case base $i_\star$ will be played roughly $T\underline{p}_{i_\star}$ times, and will repeat its decisions in intervals of length $\frac{1}{\underline{p}_{i_\star}}$, resulting in the following bound: 
\begin{lemma}[informal]\label{lemma::bounding_termII_informal}
If the regret of the optimal base is $(U_*, T, \delta)$-bounded, then we have that
$$
\mathbb{E}\left[\mathrm{II}\right] \leq O\left( \mathbb{E}\left[ \frac{1}{\underline{p}_{i_\star}} U_*(T\underline{p}_{i_\star}, \delta) \log T \right]+ \delta T(\log T +1)\right). 
$$
\end{lemma}
We demonstrate the effectiveness of our smoothing transformation by deriving regret bounds with two meta-algorithms: the Log-Barrier-OMD algorithm in CORRAL (introduced by \citep{DBLP:conf/colt/AgarwalLNS17}) which we will henceforth refer to as the CORRAL and EXP3.P (Theorem 3.3 in \citep{DBLP:journals/corr/abs-1204-5721}) with forced exploration. The later is a simple algorithm that ensures each base is picked with at least a (horizon dependent) constant probability $p$. We now state an informal version of our main result, Theorem~\ref{thm:meta-algorithm}. 

\begin{theorem}[informal version of Theorem~\ref{thm:meta-algorithm}]
\label{thm:meta-algorithm_informal}
If $U_*(T, \delta) = O(c(\delta)\,T^{\alpha})$ for some function $c:\Real\rightarrow \Real$ and constant $\alpha \in [1/2,1)$ and $\mathcal{B}_*$ is $(U_\star, T, \delta)$-bounded, the regrets of Stochastic CORRAL with an EXP3.P and CORRAL meta-algorithms are:
\begin{center}
\begin{tabular}{ |c|c|c| } 
 \hline
Meta-Algorithm & Known $\alpha$ and $c(\delta)$ & Known $\alpha$, Unknown $c(\delta)$ \\ \hline
EXP3.P &  $\widetilde{O}\left(T^{\frac{1}{2-\alpha}}  c(\delta)^{\frac{1}{2-\alpha}}\right)$ & $\widetilde{O}\left(T^{\frac{1}{2-\alpha}}  c(\delta)\right)$ \\ \hline
CORRAL  & $\widetilde{O}\left( T^{\alpha} c(\delta) \right)$  & $\widetilde{O}\left( T^{\alpha} c(\delta)^{\frac{1}{\alpha}} \right)$\\ \hline
\end{tabular}
\end{center}

\end{theorem}
The CORRAL meta-algorithm achieves optimal regret when $\alpha$ and $c(\delta)$ are known. When  $c(\delta)$ is unknown and $c(\delta) > T^{\frac{(1-\alpha)\alpha}{2-\alpha}} $ (which is $T^{1/6}$ if $\alpha=1/2$ or $\alpha =1/3$), then using an EXP3.P meta-algorithm achieves better regret because $\widetilde{O}\left(T^{\frac{1}{2-\alpha}}  c(\delta)\right) < \widetilde{O}\left( T^{\alpha} c(\delta)^{\frac{1}{\alpha}} \right)$. We complement this result with a couple of lower bounds.

\paragraph*{Lower bounds.} Theorem~\ref{thm:lower_bound2} in Section~\ref{sec:lower_bound} shows that if the regret of the best base is $O(T^{x})$, in the worst case a meta-algorithm that does not know $x$ can have regret $\Omega(T^{y})$ with $y > x$. Theorem~\ref{thm:lowerbound} shows that in general it is impossible for any meta-algorithm to achieve a regret better than $\Omega(\sqrt{T})$ even when the best base has regret $O(\log(T)$). When the regret of the best base is $O(\sqrt{T})$, CORRAL with our smoothing achieves the optimal $O(\sqrt{T})$ regret. %

The detailed description of the aforementioned smoothing procedure, its properties and the regret analysis are postponed to Section~\ref{sec::base_smoothing}. We also show some applications of our model selection results in Section~\ref{sec:applications}. 
\subsection*{Meta-Algorithms}
We review the adversarial bandit algorithms used as a Meta-Algorithm in Algorithm~\ref{Alg:meta-algorithm}. 
\subsection*{CORRAL Meta-Algorithm}
We reproduce the CORRAL meta-algorithm below. 

\begin{algorithm}[H]
\textbf{Input:} Base Algorithms $\{\mathcal{B}_j\}_{j=1}^M$, learning rate $\eta$.  \\
Initialize: $\gamma = 1/T, \beta = e^{\frac{1}{\ln T}}, \eta_{1,j} = \eta, \rho^j_{1} = 2M, \underline{p}^j_{1} = \frac{1}{ \rho^j_{1}},  {p}^j_1= 1/M$ for all $j \in [M]$. 
\For{$t = 1, \cdots, T$}{
Sample $i_t \sim p_t$. \\
Receive feedback $r_t$ from base $\mathcal{B}_{i_t}$. \\ 
Update $p_{t}$, $\eta_t$, $\underline{p}_t$ and $\rho_t$ to $p_{t+1}$, $\eta_{t+1}$, $\underline{p}_{t+1}$ and $\rho_{t+1}$ using $\mathsf{CORRAL-Update}$ Algorithm~\ref{Alg:corral_update}. \\
}
\caption{CORRAL Meta-Algorithm}
\label{Alg:corral_meta-algorithm}
\end{algorithm}

\subsection*{EXP3.P Meta-Algorithm}
We reproduce the EXP3.P algorithm (Figure 3.1 in \citep{Bubeck-Slivkins-2012}) below. In this formulation we use $\eta = 1, \gamma = 2 \beta k$ and $p = \frac{\gamma}{k}$. 

\begin{algorithm}[H]
\textbf{Input:} Base Algorithms $\{\mathcal{B}_j\}_{j=1}^M$, exploration rate $p$.  \\
Initialize: ${p}^j_{1}= 1/M$ for all $j \in [M]$. 

\For{$t = 1, \cdots, T$}{
Sample $i_t \sim p_t$. \\
Receive feedback $r_t$ from base $\mathcal{B}_{i_t}$. \\ 
Compute the estimated gain for each base $j$: 
$\widetilde{r}_{t,j} = \frac{r_{t,j}{\mathbf{1}_{i_t = j} + p/2}}{p_{j,t}}$
and update the estimated cumulative gain $\widetilde{R}_{j,t} = \sum_{s=1}^t \widetilde{r}_{s,j}$. 
\For{$j = 1, \cdots, M$}{
$p^j_{t+1} = (1 - p) \frac{\exp{\widetilde{R}_{j,t}}}{\sum_{n=1}^M\exp{\widetilde{R}_{n,t}}} + p$
}

}
\caption{EXP3.P Meta-Algorithm}
\label{Alg:exp3_meta-algorithm}
\end{algorithm}

 \section{Smoothed Algorithm and Regret Analysis}
\label{sec::base_smoothing}

\subsection*{Non-increasing Instantaneous Regret} 
We introduce a "smoothing" procedure (Algorithm~\ref{Alg:smoothing}) which, given a $(U, \delta, T)-$bounded algorithm $\mathcal{B}$ constructs a smoothed algorithm $\widetilde{\mathcal{B}}$ with the property that for some time-steps its conditional expected instantaneous regret has a decreasing upper bound. For ease of presentation and instead of making use of odd and even time-steps in the definition of $\widetilde{\mathcal{B}}$ we assume each round $t$ is split in two types of steps (Step 1 and Step 2). We will denote objects pertaining to the $t-$th round step $i$ using a subscript $t$ and a superscript $(i)$. The construction of $\widetilde{\mathcal{B}}$ is simple. The smoothed algorithm maintains an internal copy of the original algorithm $\mathcal{B}$. During step 1 of round $t$, $\widetilde{\mathcal{B}}$ will play the action suggested by its internal copy of $\mathcal{B}$. During step 2 of round $t$, $\widetilde{\mathcal{B}}$ will instead sample uniformly from the set of policies previously played by the copy of $\mathcal{B}$ maintained by $\widetilde{\mathcal{B}}$ during steps of type $1$ from round $1$ to round $t$. 

Let's define step 2 more formally. If algorithm $\mathcal{B}$ is at state $s$ during round $t$, at step 2 of the corresponding time-step the smoothed strategy will pick an index $q$ in $[1,2, .., s]$ uniformly at random, and will then re-play the policy $\mathcal{B}$ used during step $1$ of round $q$. Since $\mathcal{B}$ is $(U, \delta, T)$-bounded we will show in Lemma~\ref{lemma::martingale_smoothing} that the expected instantaneous regret of step 2 at round $s$ is at most $U(s,\delta)/s$ with high probability.

\begin{algorithm}[H]
\textbf{Input:} Base Algorithm $\mathcal{B}$;\\
Let $\pi_{s}$ be the policy of $\mathcal{B}$ in state $s$.\\
Let $\widetilde{\pi}_{s}^{(1)}, \widetilde{\pi}_{s}^{(2)}$ be the policies of $\widetilde{\mathcal{B}}$ in state $s$ during step $1$ and $2$ respectively. \\ 
Initialize state counter $s = 1$. \\
\For{ $t = 1, \cdots, T$ }{

\begin{tabular}{ c|l }
&Receive action set $\mathcal{A}^{(1)}_t \sim D_S$\\
& Let $\widetilde{\pi}_{s}^{(1)} = \pi_{s}$ from $\mathcal{B}_{i}$. \\
 \textbf{Step 1} &   Play action $a_{t}^{(1)} \sim \widetilde{\pi}^{(1)}_{s}(\mathcal{A}_t^{(1)})$.\\
&Receive feedback $r_{t}^{(1)} = f(\mathcal{A}_t^{(1)}, \delta_{a_{t}^{(1)}}) + \xi_t^{(1)}$ \\ 
  &Calculate $\pi_{s+1}$ of $\mathcal{B}$ using $r_{t}^{(1)}$. \\
\hline
\end{tabular}
\newline
\begin{tabular}{ c|l }
  & Receive action set $\mathcal{A}^{(2)}_t \sim D_S$.\\ 
  & Sample $q \sim \text{Uniform}(0, \cdots, {s})$; Let $\widetilde{\pi}_{s}^{(2)} = \pi_{q}$ from $\mathcal{B}$. \\
 \textbf{Step 2} &   Play action $a_{t}^{(2)} \sim \widetilde{\pi}^{(2)}_{s}(\mathcal{A}_t^{(2)})$.\\
 &Receive feedback $r_{t}^{(2)} = f(\mathcal{A}_t^{(2)}, \delta_{a_{t}^{(2)}}) + \xi_t^{(2)}$. \\ 
\end{tabular}\\
Update $\mathcal{B}$'s internal counter $s \leftarrow s + 1$}
\caption{Smoothed Algorithm}
\label{Alg:smoothing}
\end{algorithm}

It is easy to see that if algorithm $\widetilde{\mathcal{B}}$ has been played $\ell$ times (including step 1 and 2 plays), the internal counter of $\mathcal{B}$ equals $\ell/2$. We will make use of this internal counter when we connect a smoothed algorithm with the Stochastic CORRAL meta-algorithm. We now introduce the definition of  $(U, \delta, \mathcal{T}^{(2)})-$Smoothness which in short corresponds to algorithms that satisfy a high probability conditional expected regret upper bound during steps of type $2$.

\begin{definition}[$(U, \delta, \mathcal{T}^{(2)})-$Smoothness]\label{definition::stability}
Let $U:\mathbb{R} \times [0,1] \rightarrow \mathbb{R}^+$. We say a smoothed algorithm $\widetilde{\mathcal{B}}$  is $(U, \delta, \mathcal{T}^{(2)})-$smooth if with probability $1-\delta$ and for all rounds $t \in [T]$, the conditional expected instantaneous regret of Step 2 is bounded above by $U(t, \delta)/t$:%
\begin{equation}\label{equation::inst_regret_boundedness}
    \mathbb{E}_{\mathcal{A}'; \sim \mathcal{D_S}, \pi_t^{(2)} = \pi_q \text{ s.t. } q\sim \mathrm{Uniform}(0, \cdots, s) } [ f(\mathcal{A}', \pi^*) - f(\mathcal{A}', \pi_t^{(2)})| \widetilde{\mathcal{F}}_{t-1} ] \leq \frac{U(t, \delta)}{t}, \text{ } \forall t \in [T].
\end{equation}
Where $\widetilde{\mathcal{F}}_{t-1} = \sigma\left( \{ \mathcal{A}_\ell^{(i)},\widetilde{\pi}_{\ell}^{(i)}, r_\ell^{(i)}, a_\ell^{(i)}\}_{\ell\in [t-1], i \in \{1,2\}}, \cup \{\mathcal{A}_\ell^{(1)},\widetilde{\pi}_{\ell}^{(1)}, r_\ell^{(i)}, a_\ell^{(1)}\} \right)$ is the sigma algebra generated by all contexts, rewards, policies and actions up to time $t$ step 1. 
\end{definition}

During all steps of type 2 algorithm $\widetilde{\mathcal{B}}$ replays the policies it has used when confronted with contexts $\mathcal{A}_1^{(1)}, ..., \mathcal{A}_s^{(1)}$. In Lemma~\ref{lemma::martingale_smoothing} we will use the fact that all contexts are assumed to be generated as i.i.d. samples from the same context generating distribution $\mathcal{D}_S$ to show that $\widetilde{\mathcal{B}}$ is $(U, \delta, \mathcal{T}^{(2)})-$smooth.

With this objective in mind let's analyze a slightly more general setting. Let $\mathcal{B}$ be a $(U, \delta, T)-$bounded algorithm playing in an environment where the high probability regret upper bound $U$ holds. Let's assume that $\mathcal{B}$ has been played for $t$ time-steps during which it has encountered i.i.d. generated contexts $\mathcal{A}_1, \cdots, \mathcal{A}_t$ and has played actions sampled from policies $\pi_1, \cdots, \pi_t$. Similar to the definition of $\widetilde{\mathcal{F}}_{t-1}$ in Definition~\ref{definition::stability}, let's define $\mathcal{F}_{t-1} = \sigma\left( \{ \mathcal{A}_\ell,\widetilde{\pi}_{\ell}, r_\ell, a_\ell\}_{\ell\in [t-1]} \right)$, the sigma algebra generated by all contexts, rewards, policies and actions up to time $t-1$.  We define the ``expected replay regret" $\mathsf{Replay}(t|\mathcal{F}_{t-1})$ as:
\begin{equation}\label{equation::replay_expected_regret}
    \mathsf{Replay}(t|\mathcal{F}_{t-1}) = \mathbb{E}_{\mathcal{A}_1', \cdots, \mathcal{A}_t'}\left[ \sum_{l=1}^t f(\mathcal{A}_l', \pi^*) - f(\mathcal{A}_l', \pi_l) \right] 
\end{equation}
Where $\mathcal{A}_1', \cdots, \mathcal{A}_t'$ are i.i.d. contexts from $\mathcal{D}_S$ all of them conditionally independent from $\mathcal{F}_t$. It is easy to see that the conditional instantaneous regret of a smoothed algorithm $\widetilde{\mathcal{B}}$ during round $t$ step 2 equals the expected replay regret $\mathsf{Replay}(t|\widetilde{\mathcal{F}}_{t-1})$ of the $\mathcal{B}$ copy inside $\widetilde{\mathcal{B}}$. 

As a first step in proving that $\widetilde{\mathcal{B}}$ is $(U, \delta, \mathcal{T}^{(2)})-$smooth in Lemma~\ref{lemma::martingale_smoothing} we show the replay regret of a $(U, \delta, T)$-bounded algorithm satisfies a high probability upper bound.

\begin{lemma}\label{lemma::martingale_smoothing}
 If $\mathcal{B}$ is $(U, \delta, T)-$bounded with $U(t, \delta) > 8\sqrt{t\log(\frac{t^2}{\delta})}$ and the rewards satisfy Assumption~\ref{assumption:unit_bounded_reward}, then with probability at least $1-\delta$ for all $t \in [T]$ the expected replay regret of $\mathcal{B}$ satisfies:
 \begin{equation*}
     \mathsf{Replay}(t|\mathcal{F}_{t-1}) \leq 4 U(t, \delta) + 2 \delta t .
 \end{equation*}
 Furthermore, if $\delta \leq \frac{1}{\sqrt{T}}$ then $ \mathsf{Replay}(t|\mathcal{F}_{t-1})  \leq 5 U(t, \delta)$.
\end{lemma}

\begin{proof} Let's condition on the event $\mathcal{E}_1$ that $\mathcal{B}$'s plays satisfy the high probability regret guarantee given by $U$:
\begin{equation}\label{equation::regret_upper_bound_condition}
    \sum_{l=1}^t f(\mathcal{A}_l, \pi^*) - f(\mathcal{A}_l, \pi_l) \leq U(t, \delta).
\end{equation}
For all $t \in [T]$ and where $\mathcal{A}_1, \cdots, \mathcal{A}_t$ are the contexts algorithm $\mathcal{B}$ encountered up to time $t$. Since $\mathcal{B}$ is $(U, \delta, T)-$bounded it must be the case that $\mathbb{P}\left( \mathcal{E}_1\right) \geq 1-\delta$.

Let $\mathcal{A}_1', \cdots, \mathcal{A}_t'$ be a collection of $t$ fresh i.i.d. contexts from $\mathcal{D}_S$ independent from  $\mathcal{F}_{t}$. We now use martingale concentration arguments to show that $\sum_{l=1}^t f(\mathcal{A}_l, \pi^*) \approx  \sum_{l=1}^t f(\mathcal{A}_l', \pi^*)$ and $\sum_{l=1}^t f(\mathcal{A}_l, \pi_l) \approx \sum_{l=1}^t f(\mathcal{A}_l', \pi_l)$. Consider the following two martingale difference sequences:
\begin{align*}
&\left\{M_l^1 := f(\mathcal{A}_l, \pi^*) - f(\mathcal{A}'_l, \pi^*) \right\}_{l=1}^T \\
&\left\{M_l^2 := f(\mathcal{A}_l', \pi_l) - f(\mathcal{A}_l, \pi_l) \right\}_{l=1}^T
\end{align*}
Since by assumption $\max_{\mathcal{A}', \pi} \left|f(\mathcal{A}', \pi) \right| \leq 1$ each term in $\{ M_l^1\}$ and $\{M_l^2\}$ is bounded and satisfies $\max\left( |M_l^1|, |M_l^2| \right) \leq 2$ for all $t$. A simple use of Azuma-Hoeffding yields:
\begin{align*}
 \mathbb{P}\left(\left| \sum_{l=1}^t M_l^i \right| \geq \sqrt{8 t \log\left( \frac{8t^2}{\delta } \right) }  \right) \leq 2 \exp\left( - \frac{8t \log(\frac{8t^2}{\delta })  }{ 8t  } \right)
= \frac{ \delta}{4t^2} \;.
 \end{align*}
Summing over all $t$, and all $i \in \{1, 2\}$, using the fact that $\sum_{t=1}^T \frac{1}{t^2} < 2$ and the union bound implies that for all $t$, with probability $1-\delta$,
\begin{align}
    \left| \sum_{l=1}^t f(\mathcal{A}_l,\pi_l) -  f(\mathcal{A}_l', \pi_l)  \right| \leq \sqrt{8 t \log\left( \frac{8t^2}{\delta } \right) }\label{equation::replay_helper_1}\\
        \left|\sum_{l=1}^t f(\mathcal{A}_l,\pi^*) -  f(\mathcal{A}_l', \pi^*)  \right| \leq \sqrt{8 t \log\left( \frac{8t^2}{\delta } \right) }\label{equation::replay_helper_2}\\
\end{align}
Denote this event as $\mathcal{E}_2$. 
We shall proceed to upper bound the replay regret. Let's condition on $\mathcal{E}_1 \cap \mathcal{E}_2$. The following sequence of inequalities holds,
\begin{align}
     \sum_{l=1}^t f(\mathcal{A}'_l, \pi^*) - f(\mathcal{A}'_l, \pi_l) &\stackrel{(i)}{\leq}  \sum_{l=1}^t f(\mathcal{A}_l, \pi^*) - f(\mathcal{A}_l, \pi_l) + \left|\sum_{l=1}^t f(\mathcal{A}_l,\pi_l) -  f(\mathcal{A}_l', \pi_l) \right| + \notag \\
     &\quad \left| \sum_{l=1}^t f(\mathcal{A}_l,\pi^*) -  f(\mathcal{A}_l', \pi^*)\right |\notag\\
     &\leq U(t, \delta) + 2\sqrt{8 t \log\left( \frac{8t^2}{\delta } \right) } \notag%
\end{align}

For all $t \in [T]$. Inequality $(i)$ follows by the triangle inequality while $(ii)$ is a consequence of conditioning on $\mathcal{E}_1 \cap \mathcal{E}_2$ and invoking inequalities~\ref{equation::regret_upper_bound_condition},~\ref{equation::replay_helper_1} and~\ref{equation::replay_helper_1}. We conclude that with probability at least $1-2\delta$ and for all $t \in [T]$,

\begin{align*}
     \sum_{l=1}^t f(\mathcal{A}'_l, \pi^*) - f(\mathcal{A}'_l, \pi_l) & \leq U(t, \delta) + 2\sqrt{8 t \log\left( \frac{8t^2}{\delta } \right) }
\end{align*}
Since we have assumed that $U(t, \delta) > 8\sqrt{t\log(\frac{t^2}{\delta})}$, averaging out over the randomness in $\{ \mathcal{A}_l' \}_{l=1}^t$ yields that conditioned on $\mathcal{E}_1$,
\begin{equation*}
     \mathsf{Replay}(t|\mathcal{F}_{t-1}) \leq 4 (1-2\delta) U(t, \delta) + 2 \delta t  < 4 U(t, \delta) + 2 \delta t.
 \end{equation*}
It is easy to see that in case  $\delta \leq \frac{1}{\sqrt{T}}$ then $ \mathsf{Replay}(t|\mathcal{F}_{t-1})  \leq 5 U(t, \delta)$. 
\end{proof}

 In Propositon~\ref{prop::decreasing_instant_regret} we show that if $\mathcal{B}$ is bounded, then $\widetilde{\mathcal{B}}$ is both bounded and smooth. We will then show  that several algorithms such as UCB, LinUCB, $\epsilon$-greedy and EXP3 are $(U, \delta, T)$-bounded for appropriate functions $U$. By Proposition~\ref{prop::decreasing_instant_regret} we will then conclude the smoothed versions of these algorithms are smooth. 
\begin{proposition}
\label{prop::decreasing_instant_regret}
If  $U(t, \delta) > 8\sqrt{t\log(\frac{t^2}{\delta})}$, $\delta \leq \frac{1}{\sqrt{T}}$, the rewards satisfy Assumption~\ref{assumption:unit_bounded_reward} and $\mathcal{B}$ is $(U, \delta, T)-$bounded, then $\widetilde{\mathcal{B}}$ is $(5U, \delta, \mathcal{T}^{(2)})-$smooth and with probability at least $1-3\delta$,
\begin{equation*}
     \sum_{l=1}^t \sum_{i \in \{1,2\}} f(\mathcal{A}_l^{(i)}, \pi^*) - f(\mathcal{A}_l^{(i)}, \pi_l^{(i)}) \leq 7U(t, \delta) \log(t).
\end{equation*}%
for all $t \in [T]$.
\end{proposition}
\begin{proof}
Let $\mathcal{E}_1$ denote the event that $\widetilde{\mathcal{B}}$'s plays during steps of type $1$ satisfy the high probability regret guarantee given by $U$:
\begin{equation}\label{equation::supporting_upper_bound}
    \sum_{l=1}^t f(\mathcal{A}^{(1)}_l, \pi^*) - f(\mathcal{A}^{(1)}_l, \pi^{(1)}_l) \leq U(t, \delta).
\end{equation}
for all $t \in [T]$. Since the conditional instantaneous regret of Step $2$ of round $t$ equals the average replay regret of the type $1$ steps up to $t$, Lemma~\ref{lemma::martingale_smoothing} implies that whenever $\mathcal{E}_2$ holds (see definition for $\mathcal{E}_2$ in the proof of Lemma~\ref{lemma::martingale_smoothing}) which occurs with probability at least $1-\delta$, the conditional expected instantaneous regret satisfies: $\E[ f(\mathcal{A}', \pi^*) - f(\mathcal{A}', \pi_t^{(2)})| \widetilde{\mathcal{F}}_{t-1}] \leq \frac{5U(t, \delta)}{t}$ for all $t \in [T]$. This shows that $\widetilde{\mathcal{B}}$ is $(5U, \delta, \mathcal{T}^{(2)})-$smooth.

It is easy to see that if we condition on $\mathcal{E}_1 \cap \mathcal{E}_2$ the conditional expected instantaneous regret of steps of type $2$ satisfy,
\begin{equation}\label{equation::upper_bounding_expected_instantaneous_regrets}
    \sum_{l=1}^t \E[ f(\mathcal{A}', \pi^*) - f(\mathcal{A}', \pi_l^{(2)})| \widetilde{\mathcal{F}}_{l-1}] \leq \sum_{l=1}^t \frac{5U(l, \delta)}{l} \leq 5U(t, \delta) \log(t)
\end{equation}

For all $t \in [T]$. We now show the regret incurred by $\widetilde{\mathcal{B}}$ satisfies a high probability upper bound.  To bound the regret accrued during time-steps of type $2$, consider the following Martingale difference sequences,
\begin{align*}
\left\{     M_l^1 :=  \E[ f(\mathcal{A}', \pi_l^{(2)})| \widetilde{\mathcal{F}}_{l-1}]  - f(\mathcal{A}_l^{(2)}, \pi_l^{(2)} )  \right\}_{l=1}^T \\
\left\{     M_l^2 :=  \E[ f(\mathcal{A}', \pi^*)| \widetilde{\mathcal{F}}_{l-1}]  - f(\mathcal{A}_l^{(2)}, \pi^* )  \right\}_{l=1}^T 
\end{align*}
As a result of Assumption~\ref{assumption:unit_bounded_reward}, $| M^i_l | \leq 2$ for all $i \in \{1, 2\}$ and therefore a simple use of Azuma-Hoeffding's inequality,
\begin{align*}
 \mathbb{P}\left(\left| \sum_{l=1}^t M^i_l \right| \geq \sqrt{8 t \log\left( \frac{8t^2}{\delta } \right) }  \right)  \leq 2 \exp\left( - \frac{8t \log(\frac{8t^2}{\delta })  }{ 8t  } \right) = \frac{ \delta}{4t^2} \;.
 \end{align*}
Summing over all $t$, applying the union bound, using the fact that $\sum_{t=1}^T \frac{1}{t^2} < 2$ implies that for all $t \in [T]$, with probability $1-\delta$,
\begin{align}
    \left| \sum_{l=1}^t \E[ f(\mathcal{A}', \pi^*) - f(\mathcal{A}', \pi_l^{(2)})| \widetilde{\mathcal{F}}_{l-1}]  - f(\mathcal{A}_l^{(2)}, \pi^*) - f(\mathcal{A}_l^{(2)}, \pi_l^{(2)}) \right| &\leq \sqrt{8 t \log\left( \frac{8t^2}{\delta } \right) } \notag \\
    &\leq U(t, \delta) \label{equation::bounding_average_regret_vs_empirical_smoothing}
\end{align}
Let's denote as $\mathcal{E}_3$ the event where Equation~\ref{equation::bounding_average_regret_vs_empirical_smoothing} holds. 
If $\mathcal{E}_2 \cap \mathcal{E}_3$ occur, then combining the upper bounds in~\ref{equation::upper_bounding_expected_instantaneous_regrets} and~\ref{equation::bounding_average_regret_vs_empirical_smoothing} we conclude that,
\begin{align*}
    \sum_{l=1}^t  f(\mathcal{A}_l^{(2)}, \pi^*) - f(\mathcal{A}_l^{(2)}, \pi_l^{(2)}) \leq 6U(t, \delta)\log(t)
\end{align*}
combining this last observation with Equation~\ref{equation::supporting_upper_bound}, we conclude that for all $t$ with probability at least $1-3\delta$,
\begin{equation*}
    \sum_{l=1}^t \sum_{i \in \{1,2\}} f(\mathcal{A}_l^{(i)}, \pi^*) - f(\mathcal{A}_l^{(i)}, \pi_l^{(i)})  \leq 7 U(t, \delta,)\log(t)
\end{equation*}
For all $t \in [T]$. The result follows.
\end{proof}

It remains to show how to adapt the feedback structure of the Stochastic CORRAL meta-algorithm to deal with the two step nature of smoothed algorithms. We reproduce the full pseudo-code of the Stochastic CORRAL meta-algorithm adapted to smoothed algorithms below,

\begin{algorithm}[H]
 \textbf{Input:} Smoothed Base Algorithms $\{\widetilde{\mathcal{B}}_j\}_{j=1}^M$, bias functions $\{ b_j: \mathbb{N} \rightarrow \mathbb{R} \}_{j=1}^M$ \\
\For{ $t = 1, \cdots, T$ }{ 
 Sample $j_t \sim p_t$.\\
 Play base algorithm $j_t$ for Steps 1 and 2. \\ 
Receive feedback $r_t^{(1)}$ and $r_t^{(2)}$ from Steps $1$ and $2$ when executing $\widetilde{\mathcal{B}}_{j_t}$. \\
Let $s_{t,j_t}$ be the internal counter at time $t$ of $\widetilde{\mathcal{B}}_{j_t}$ as defined in Algorithm~\ref{Alg:smoothing}.\\
Update $p_t$ using $2r_t^{(2)} - b_{j_t}(s_{t,j_t})$
 }
 \caption{Smooth Stochastic CORRAL Meta-Algorithm}
\label{Alg:smooth_meta-algorithm}
\end{algorithm}

For reasons that have to do with the analysis, Algorithm~\ref{Alg:smooth_meta-algorithm} has a few extra features not present in the meta-algorithm-base template of Algorithm~\ref{Alg:meta-algorithm}. First, whenever the smooth stochastic corral meta-algorithm selects an algorithm $j_t$ it plays it for two steps, thus coinciding with $\widetilde{\mathcal{B}}_{j_t}$'s two time step structure. Second, it updates its distribution $p_t$ using the feedback $2r_t^{(2)} - b_{j_t}(s_{t,j_t})$ instead of using the sum $r_t^{(1)} + r_t^{(2)}$. Most notably, the update makes use of a bias adjustment to the reward signal that is not present in the original. The reason behind this modification will become clearer in the regret analysis.%

\subsection*{Applications of Proposition~\ref{prop::decreasing_instant_regret}}
We now show the smoothed versions of several algorithms satisfy Definition~\ref{definition::stability} by showing they are $(U,\delta,T)-$bounded for an appropriate upper bound function $U$. We focus on algorithms for the $k-$armed bandit setting and the contextual linear bandit setting.
\begin{lemma}[Theorem~3 in \citep{abbasi2011improved}]
\label{lem::Lin_UCB} In the case of changing and potentially infinite contexts of dimension $d$, LinUCB is $(U, \delta, T)$-bounded with $U(t, \delta) = O(d\sqrt{t}\log(1/\delta))$. 
\end{lemma}
\begin{lemma}[Theorem~1 in \citep{pmlr-v15-chu11a}]
\label{lem::Lin_UCB2}
 In the case of finite linear contexts of size $k$ and dimension $d$, LinUCB is $(U, \delta, T)$-bounded with $U(t, \delta) = O(\sqrt{dt}\log^3(kT\log(T)/\delta))$.
\end{lemma}

\begin{lemma}[Theorem 1 in \citep{pmlr-v24-seldin12a}]
In the $k-$armed adversarial bandit setting Exp3 is $(U, \delta, T)-$bounded where $U(t,\delta) = O(\sqrt{tk}\log\frac{tk}{\delta})$. 
\end{lemma}
\begin{lemma}
\label{lem::UCB}
In the stochastic $k-$armed bandit problem, if we assume the noise $\xi_t$ is conditionally 1-sub-Gaussian, UCB is $(U, \delta, T)$-bounded with $U(t,\delta) = O(\sqrt{tk}\log\frac{tk}{\delta})$. 
\end{lemma}
\begin{proof}
The regret of UCB is bounded as $\sum_{i:\Delta_i > 0} \left ( 3\Delta_i + \frac{16}{\Delta_i}\log\frac{2k}{\Delta_i \delta} \right ) $ (Theorem~7 of \citep{abbasi2011improved}) where $\Delta_i$ is the gap between arm $i$ and the best arm. By substituting the worst-case $\Delta_i$ in the regret bound, $U(T, \delta) = O(\sqrt{Tk}\log\frac{Tk}{\delta})$.  
\end{proof}

For the remainder of this section we focus on showing that in the stochastic $k-$armed bandit problem, the $\epsilon$-greedy algorithm (Algorithm 1.2 of~\cite{slivkins2019introduction}) is $(U, T, \delta)-$bounded. At time $t$ the $\epsilon$-greedy algorithm selects with probability $\epsilon_t = \min(c/t, 1)$ an arm uniformly at random, and with probability $1-\epsilon_t$ it selects the arm whose empirical estimate of the mean is largest so far. Let's introduce some notation: we will denote by $\mu_1, \cdots, \mu_k$ the unknown means of the $K$ arms use the name $\widehat{\mu}_j^{(t)}$ to denote the empirical estimate of the mean of arm $j$ after using $t$ samples. 

Without loss of generality let $\mu_1$ be the optimal arm. We denote the sub-optimality gaps as $\Delta_j = \mu_1 - \mu_j$ for all $j \in [k]$. Let $\Delta_*$ be the smallest nonzero gap. We follow the discussion in \citep{auer2002finite} and start by showing that under the right assumptions, and for a horizon of size $T$, the algorithm satisfies a high probability regret bound for all $t \leq T$. The objective of this section is to prove the following Lemma:

\begin{lemma}\label{lemma::epsilon_greedy}
If $c=\frac{10K\log(T^3/\gamma)}{\Delta_*^2} $\footnote{This choice of $c$ is robust to multiplication by a constant.} for some $\gamma \in (0,1)$ satisfying $\gamma \leq \frac{\Delta_j^2}{2}$, then $\epsilon-$greedy with $\epsilon_t = \frac{c}{t}$ is $(U,\delta, T)-$bounded for $\delta \leq \frac{\Delta_*^2}{T^3}$ where $$U(t, \delta) = \frac{30 k \log(\frac{1}{\delta})}{\Delta_*^2}\left(\sum_{j=2}^k  \frac{\Delta_j}{\Delta_*^2} + \Delta_j \right)\log(t+1).$$   
\end{lemma}

\begin{proof}

Let $E(t)  =  \frac{1}{2k}\sum_{l=1}^t \epsilon_l$ and denote by $T_j(t)$ the random variable denoting the number of times arm $j$ was selected up to time $t$. We start by analyzing the probability that a suboptimal arm $j > 1$ is selected at time $t$:

\begin{equation}
    \mathbb{P}( j \text{ is selected at time }t ) \leq \frac{\epsilon_t}{k} +\left( 1-\frac{\epsilon_t}{k}\right) \mathbb{P}\left(   \widehat{\mu}_j^{(T_j(t))} \geq \widehat{\mu}_1^{(T_1(t))}   \right)
\end{equation}
Let's bound the second term. 

\begin{align*}
    \mathbb{P}\left( \widehat{\mu}_j^{(T_j(t))} \geq \widehat{\mu}_1^{(T_1(t))}   \right)  \leq \mathbb{P}\left(  \widehat{\mu}_j^{(T_j(t))} \geq \mu_j + \frac{\Delta_j}{2}   \right) + \mathbb{P}\left(  \widehat{\mu}_1^{(T_1(t))} \leq \mu_1 - \frac{\Delta_j }{2}  \right)
\end{align*}

The analysis of these two terms is the same. Denote by $T_j^R(t)$ the number of times arm $j$ was played as a result of a random epsilon greedy move. We have:
\begin{small}
\begin{align*}
    \mathbb{P}\left(     \widehat{\mu}_j^{(T_j(t))} \geq \mu_j + \frac{\Delta_j}{2}     \right) &= \sum_{l=1}^t \mathbb{P}\left( T_j(t) = l \text{ and } \widehat{\mu}_j^{(l)} \geq \mu_j + \frac{\Delta_j}{2}  \right) \\
    &= \sum_{l=1}^t \mathbb{P}\left( T_j(t) = l | \widehat{\mu}_j^{(l)} \geq \mu_j + \frac{\Delta_j}{2}  \right) \mathbb{P}\left(  \widehat{\mu}_j^{(l)} \geq \mu_j + \frac{\Delta_j}{2}  \right) \\
    &\stackrel{a}{\leq } \sum_{l=1}^t \mathbb{P}\left( T_j(t) = l \Big| \widehat{\mu}_j^{(l)} \geq \mu_j + \frac{\Delta_j}{2}  \right) \exp(-\Delta_j^2 t/2) \\
    &\stackrel{b}{\leq} \sum_{l=1}^{\lfloor E(t)\rfloor} \mathbb{P}\left( T_j(t) = l \Big| \widehat{\mu}_j^{(l)} \geq \mu_j + \frac{\Delta_j}{2}  \right) + \frac{2}{\Delta_j^2}\exp(-\Delta_j^2 \lfloor E(t) \rfloor /2) \\
    &\leq \sum_{l=1}^{\lfloor E(t)\rfloor} \mathbb{P}\left( T_j^R(t) = l \Big| \widehat{\mu}_j^{(l)} \geq \mu_j + \frac{\Delta_j}{2}  \right) + \frac{2}{\Delta_j^2}\exp(-\Delta_j^2 \lfloor E(t) \rfloor /2)\\
    &\leq \underbrace{\lfloor E(t) \rfloor \mathbb{P}\left(   T_j(t)^R \leq \lfloor E(t) \rfloor   \right)}_{(1)} + \underbrace{ \frac{2}{\Delta_j^2}\exp(-\Delta_j^2 \lfloor E(t) \rfloor /2)}_{(2)}
\end{align*}
\end{small}
Inequality $a$ is a consequence of the Azuma-Hoeffding inequality bound. Inequality $b$ follows because $\sum_{l=E+1}^\infty \exp(-\alpha l) \leq \frac{1}{a} \exp( -\alpha E )$. Term $(1)$ corresponds to the probability that within the interval $[1, \cdots, t]$, the number of greedy pulls to arm $j$ is at most half its expectation. Term $(2)$ is already "small". Lets proceed to bound $(1)$.  Let $\epsilon_t = \min(c/t, 1)$. with $c=\frac{10k\log(T^3/\gamma)}{\Delta_*^2}$ for some $\gamma \in (0,1)$ satisfying $\gamma \leq \Delta_j^2$. We'll show that under these assumptions we can lower bound $E(t)$. If $t \geq \frac{10k\log(T^3/\gamma)}{\Delta_*^2}$:
\begin{align*}
   E(t):= \frac{1}{2k} \sum_{l=1}^t \epsilon_l&= \frac{5\log(T^3/\gamma)}{\Delta_*^2}+ \frac{5\log(T^3/\delta)}{\Delta_*^2} \sum_{l=\log(T^3/\gamma )}^t \frac{1}{l}\\
    &\geq \frac{5\log(T^3/\gamma)}{\Delta_*^2} + \frac{5\log(T^3/\gamma)\log(t) }{2\Delta_*^2}\\
    &\geq \frac{5\log(T^3/\gamma)}{\Delta_*^2}
\end{align*}
By Bernstein's inequality (see derivation of equation (13) in \citep{auer2002finite}) we can upper bound $T_j^R(t)$:%
\begin{equation}\label{equation::bernstein_bound_epsilon_greedy}
    \mathbb{P}\left(  T_j^R(t) \leq E(t)    \right) \leq  \exp\left(   -E(t)/5   \right) 
\end{equation}
Hence for $t \geq \frac{10k\log(T^3/\gamma)}{\Delta_*^2}$:
\begin{equation*}
    \mathbb{P}\left(  T_j^R(t) \leq E(t)    \right) \leq  \left(\frac{\gamma}{T^3}\right)^{\frac{1}{\Delta_*^2}}
\end{equation*}
And therefore since $E(t) \leq T$  and $\frac{1}{\Delta_*} \geq 1$ we can upper bound $(1)$ as:
\begin{equation*}
   \lfloor E(t) \rfloor \mathbb{P}\left(   T^R_j(t) \leq \lfloor E(t) \rfloor   \right) \leq \left(\frac{\gamma}{T^2}\right)^{\frac{1}{\Delta_*^2}} \leq \frac{\gamma}{T^2}
\end{equation*}
Now we proceed with term $(2)$:
\begin{align*}
\frac{2}{\Delta_j^2}\exp\left(  -\Delta_j^2 \lfloor E(t) \rfloor/2   \right) &\stackrel{(a)}{\leq} \frac{2}{\Delta_j^2} \exp\left( -\frac{5}{2} \log\left(\frac{T^3}{\gamma}\right) \frac{\Delta_j^2}{\Delta_*^2}    \right) \\
&\leq \frac{2}{\Delta_j^2} \exp\left( - \log\left(\frac{T^3}{\gamma}\right)     \right)\\
&= \frac{2}{\Delta_j^2}\left( \frac{\gamma}{T^3}     \right)^5 \\
&\stackrel{(b)}{\leq} \frac{\gamma}{T^3}  
\end{align*}

The first inequality $(a)$ follows because $E(t) \geq \frac{5\log(T^3/\gamma)}{\Delta_*^2}$. Inequality $(b)$ follows because by the assumption $\gamma \leq \frac{\Delta_j^2}{2}$ the last term is upper bounded by $\frac{\gamma}{T^3}$.

By applying the union bound over all arms $j \neq 1$ and time-steps $t \geq  \frac{10k\log(T^3/\gamma)}{\Delta_*^2}$, we conclude that the probability of choosing a sub-optimal arm $j \geq 2$ at any time time $t$ for $t \geq \frac{10k\log(T^3/\gamma)}{\Delta_*^2}$ as a \textbf{greedy choice} is upper bounded by $\frac{k
\gamma}{T^2}\leq \frac{k\gamma }{T}$. In other words after $t \geq \frac{10k\log(T^3/\gamma)}{\Delta_*^2}$ rounds, with probability $1-\frac{k\gamma}{T}$ sub-optimal arms are only chosen as a result of random epsilon greedy move (occurring with probability $\epsilon_t$). 

A similar argument as the one that gave us Equation \ref{equation::bernstein_bound_epsilon_greedy} can be used to upper bound the probability that $T^R_j(t)$ be much larger than its mean:
\begin{equation*}
    \mathbf{P}\left( T_j^R(t) \geq 3E(j)    \right) \leq \exp(-E(t)/5)
\end{equation*}
Using this and the union bound we see that with probability more than $1-\frac{k\gamma}{T}$ and for all $t \in [T]$ and arms $j \in [k]$,  $T_j^R(t) \leq 3E(t)$. 
Combining this with the observation that after $t\geq\frac{10k\log(T^3/\gamma)}{\Delta_*^2} $ and with probability $1-
\frac{k\gamma}{T}$ over all $t$ simultaneously regret is only incurred by random exploration pulls (and not greedy actions), we can conclude that with probability at least $1-\frac{2k\gamma}{T}$ simultaneously for all $t \geq \frac{10k\log(T^3/\gamma)}{\Delta_*^2}$ the regret incurred is upper bounded by:
\begin{equation*}
    \underbrace{ \frac{10k\log(T^3/\gamma)}{\Delta_*^2}\cdot \frac{1}{k}\sum_{j=2}^k \Delta_j }_{(i)} +\underbrace{ 3E(t)\sum_{j=2}^k \Delta_j }_{(ii)}
\end{equation*}
Term $(i)$ is a crude upper bound on the regret incurred in the first $\frac{10k\log(T^3/\gamma)}{\Delta_*^2} $ rounds and $(ii)$ is an upper bound for the regret incurred in the subsequent rounds. 

Since $E(t) \leq \frac{20k\log(T^3/\gamma)}{\Delta_*^2} \log(t)$ we conclude that with probability $1-\frac{2k\gamma}{T}$ for all $t \leq T$ the cumulative regret of epsilon greedy is upper bounded by $$30K\log(T^3/\gamma) \left(\sum_{j=2}^k  \frac{\Delta_j}{\Delta_*^2} + \Delta_j \right) \max( \log(t), 1), $$ the result follows by identifying $\delta = \gamma/T^3$.

\end{proof}
Lemma~\ref{lemma::epsilon_greedy} gives us an instance dependent upper bound for the $\epsilon-$greedy algorithm. We now show the instance-independent high probability regret bound for $\epsilon$-greedy: 
\begin{lemma}\label{lem::epsilon_greedy}
If $c = \frac{10k\log(\frac{1}{\delta})}{\Delta_*^2}$, then $\epsilon-$greedy with $\epsilon_t = \frac{c}{t}$ is $(\delta, U, T)-$bounded for $\delta \leq \frac{\Delta_*^2}{T^3}$ and:
\begin{enumerate}
    \item $U(t,\delta) = 16\sqrt{\log(\frac{1}{\delta}) t }$ when $k=2$. 
    \item $U(t, \delta) = 20\left( k\log(\frac{1}{\delta}) \left(\sum_{j=2 }^K  \Delta_j\right) \right)^{1/3} t^{2/3}$ when $k > 2$.
\end{enumerate}
  
\end{lemma}
\begin{proof}
Let $\Delta$ be some arbitrary gap value. Let $R(t)$ denote the expected regret up to round $t$. We recycle the notation from the proof of Lemma \ref{lemma::epsilon_greedy}, recall $\delta = \gamma/T^3$.
\begin{align}
    R(t) &= \sum_{\Delta_j \le \Delta } \Delta_j \mathbb{E}\left[T_j(t)\right] + \sum_{\Delta_j \ge \Delta } \Delta_j \mathbb{E}\left[ T_j(t)  \right] \notag \\
    &\le \Delta t + \sum_{\Delta_j \ge \Delta } \Delta_j \mathbb{E}\left[T_j(t)\right] \notag \\
    &\le \Delta t + 30k\log(T^3/\gamma) \left(\sum_{\Delta_j \ge \Delta }^k  \frac{\Delta_j}{\Delta_*^2} + \Delta_j \right)\log(t) \notag  \\
    &\le \Delta t + 30k\log(T^3/\gamma) \left(\sum_{\Delta_j \ge \Delta }^k  \frac{\Delta_j}{\Delta_*^2}\right) + 30k\log(T^3/\gamma) \log(t)\left(\sum_{\Delta_j \ge \Delta }^k \Delta_j     \right) \label{eq::k_big_than_2}
\end{align}
When $k=2$, $\Delta_2 = \Delta_*$ and therefore (assuming $\Delta < \Delta_2$):
\begin{align*}
    R(t) &\leq \Delta t +    \frac{30k\log(T^3/\gamma)}{\Delta_2} + 30k\log(T^3/\gamma) \log(t) \Delta_2 \\
    &\leq \Delta t +   \frac{30k\log(T^3/\gamma)}{\Delta} + 30k\log(T^3/\gamma)\log(t)\Delta_2 \\
    &\stackrel{\mathrm{A}}{\leq} \sqrt{30k\log(T^3/\gamma) t } + 30k\log(T^3/\gamma)\log(t)\Delta_2 \\
    &\stackrel{\mathrm{B}}{\leq} 8\sqrt{k\log(T^3/\gamma) t }\\
    &\leq 16\sqrt{\log(T^3/\gamma) t }
\end{align*}
Inequality $\mathrm{A}$ follows from setting $\Delta$ to the optimizer, which equals $\Delta = \sqrt{\frac{30k\log(T^3/\gamma)}{t} } $. The second inequality $\mathrm{B}$ is satisfied for $T$ large enough. We choose this expression for simplicity of exposition. 

When $k > 2$ notice that we can arrive to a bound similar to \ref{eq::k_big_than_2}:

\begin{equation*}
    R(t)  \leq \Delta t + 30k\log(T^3/\gamma) \left(\sum_{\Delta_j \ge \Delta }^k  \frac{\Delta_j}{\Delta^2}\right) + 30k\log(T^3/\gamma) \log(t)\left(\sum_{\Delta_j \ge \Delta }^k \Delta_j     \right) 
\end{equation*}

Where $\Delta_*$ is substituted by $\Delta$. This can be obtained from Lemma \ref{lemma::epsilon_greedy} by simply substituting $\Delta_*$ with $\Delta$ in the argument for arms $j : \Delta_j \geq \Delta$.

We upper bound $\sum_{\Delta_j \geq \Delta} \Delta_j$ by $\sum_{j=2}^k \Delta_j$. Setting $\Delta$ to the optimizer of the expression yields $\Delta = \left(    \frac{30k \log(T^3/\gamma) \left(\sum_{j = 2  }^k  \Delta_j\right)}{t} \right)^{1/3}$, and plugging this back into the equation we obtain:

\begin{align*}
    R(t) &\leq 2\left(    30k\log(T^3/\gamma) \left(\sum_{j=2 }^k  \Delta_j\right) \right)^{1/3} t^{2/3} + 30k\log(T^3/\gamma) \log(t)\left(\sum_{j=2 }^k \Delta_j     \right) \\
    &\stackrel{\xi}{\leq}  20\left( k \log(T^3/\gamma) \left(\sum_{j=2 }^k \Delta_j\right) \right)^{1/3} t^{2/3}
\end{align*}

The inequality $\xi$ is true for $T$ large enough. We choose this expression for simplicity of exposition.

\end{proof}

\subsection*{Regret Analysis}

In this section we go back to sketch the proof of Theorem~\ref{thm:meta-algorithm_informal} by explaining how to bound terms $\mathrm{I}$ and $\mathrm{II}$ in the regret decomposition of Equation~\ref{eq::regret_decomp}. 

\paragraph*{Bounding Term I.} Recall that Algorithm~\ref{Alg:smooth_meta-algorithm} only sends the smoothed reward of Step 2 to the meta-algorithm while the base plays and incurs regrets from both Step 1 and Step 2. We show in Section~\ref{app:ommited_proofs_stochastic_corral} that this does not affect the regret of the meta-algorithm significantly. For CORRAL with learning rate $\eta$,  $\EE{\mathrm{I}}  \le O\left( \sqrt{MT} + \frac{M \ln T}{\eta} + T \eta \right)- \frac{\EE{\frac{1}{\underline{p}_{i_\star}}}}{40 \eta \ln T}.$ For EXP3.P with exploration rate $p$, $\EE{\mathrm{I}} < O(\sqrt{MT} + \frac{1}{p} + MTp)$. 

\paragraph*{Bounding Term $\mathbf{II}$.} This quantity is the regret of all the policies proposed by the optimal base $i_\star$, even during steps when it was not selected by the meta-algorithm. Recall the internal state of any algorithm (including $i_\star$) is only updated when selected and played by $\mathcal{M}$. We assume the smoothed base algorithm $\widetilde{\mathcal{B}}_{i_\star}$ satisfies the smoothness and boundedness properties of Definitions~\ref{definition::boundedness} and ~\ref{definition::stability}. For the purpose of the analysis we declare that when a smoothed base repeats its policy while not played, it repeats its subsequent Step 2 policy (Algorithm~\ref{Alg:smoothing}). This will become clearer in Section~\ref{app:ommited_proofs_stochastic_corral}. Since we select $\widetilde{\mathcal{B}}_{i_\star}$ with probability at least $\underline{p}_{i_\star}$ it will be updated at most every $1/\underline{p}_i$ time-steps and the regret upper bound will be roughly $\frac{1}{\underline{p}_{i_\star}}\, U_{i_\star}(T\underline{p}_{i_\star}, \delta)$.

\begin{theorem}\label{theorem::path_dependent_regret} 
We have that $\EE{\mathrm{II}} \leq \mathcal{O}\left(\EE{ \frac{1}{\underline{p}_i}\, U_i(T\underline{p}_i, \delta) \log T }+ \delta T(\log T +1)\right)$. 
Here, the expectation is over the random variable $\underline{p}_i$. If $U(t, \delta) = t^\alpha c(\delta)$ for some $\alpha \in [1/2,1)$ then,
$\EE{\mathrm{II} }\leq \widetilde{\mathcal{O}}\left(T^\alpha c(\delta)\EE{\frac{1}{\underline{p}_i^{1-\alpha}}}+\delta T(\log T +1)\right)$. 
\end{theorem}

\label{sec::meta-algorithm}

\paragraph*{Total Regret. }Adding Term I and Term II gives us the following worst-case model selection regret bound for the CORRAL meta-algorithm (maximized over $\underline{p}_{i_\star}$ and with a chosen $\eta$) and the EXP3.P meta-algorithm (with a chosen $p$):

\begin{theorem}
\label{thm:meta-algorithm}
If a base algorithm is $(U, \delta, T)$-bounded for $U(T, \delta) = T^{\alpha}c(\delta)$ and some $\alpha \in [1/2,1)$ and the choice of $\delta = 1/T$, the regret of the Smooth Stochastic CORRAL (Algorithm~\ref{Alg:smooth_meta-algorithm}) where $b_j(s) = \frac{U_j(s, \delta)}{s}$ is upper bounded by :

\begin{table}
\begin{center}
\begin{tabular}{ |c|c| } 
 \hline
    EXP3.P & CORRAL \\ \hline
 $\widetilde{O}\left( \sqrt{MT} +  MTp  + T^{\alpha} p^{\alpha-1} c(\delta) \right)$  & $\widetilde{O}\left( \sqrt{MT}+ \frac{M }{\eta} + T\eta  + T\,c(\delta)^{\frac{1}{\alpha}} \eta^{\frac{1-\alpha}{\alpha}}\right) $ \\ \hline
 $\widetilde{O}\left(\sqrt{MT} +  M^{\frac{1-\alpha}{2-\alpha}} T^{\frac{1}{2-\alpha}}  c(\delta)^{\frac{1}{2-\alpha}}\right)$& $\widetilde{O}\left(\sqrt{MT} + M^{\alpha}T^{1-\alpha}+ M^{1-\alpha}T^{\alpha} c(\delta) \right)$ \\ \hline
 $\widetilde{O}\left(\sqrt{MT} +  M^{\frac{1-\alpha}{2-\alpha}} T^{\frac{1}{2-\alpha}}  c(\delta)\right)$ & $\widetilde{O}\left(\sqrt{MT} + M^{\alpha}T^{1-\alpha}+ M^{1-\alpha}T^{\alpha} c(\delta)^{\frac{1}{\alpha}} \right)$ \\ \hline
\end{tabular}
\end{center}

\caption[Comparison of model selection guarantees for Stochastic CORRAL between the EXP3.P and CORRAL meta-algorithm. ]{Comparison of model selection guarantees for Stochastic CORRAL between the EXP3.P and CORRAL meta-algorithm. The top row shows the general regret guarantees. The middle row shows the regret guarantees when $\alpha$ and $c(\delta)$ are known. The bottom row shows the regret guarantees when $\alpha$ is known and $c(\delta)$ is unknown.}
\end{table}

\end{theorem}

\section{Lower Bound}\label{sec:lower_bound}
In stochastic environments, algorithms such as UCB can achieve logarithmic regret bounds. Our model selection procedure however has a $O(\sqrt{T})$ overall regret. In this section, we show that in general it is impossible to obtain a regret better than $\Omega(\sqrt{T})$ even when the optimal base algorithm has $0$ regret. In order to formalize this statement, let's define a model selection problem formally. 

\begin{definition}[Model Selection Problem] We call a tuple $(\{\mathcal{B}_i\}_{i=1}^M, \mathrm{Env})$ a model selection problem where $\{\mathcal{B}_i\}_{i=1}^M$ is a set of $M$ base algorithms and $\mathrm{Env}$ is a bandit environment\footnote{For example if $M=2$ $(\{ \text{UCB}, \text{LinUCB}\},\mathrm{MAB})$ is a valid Model Selection Problem }.
\end{definition}

\begin{theorem}
\label{thm:lowerbound}
Let $T \in \mathbb{N}$. For any model selection algorithm there exists a corresponding model selection problem $(\{ \mathcal{B}_1, \mathcal{B}_2\}, \mathrm{Env})$ such the regret of this model selection algorithm is lower bounded by $R(T) = \Omega\left(\frac{\sqrt{T}}{\log(T)}\right)$. 
\end{theorem}

\begin{proof}%
Consider a stochastic $2$-arm bandit problem where the best arm has expected reward $1/2$ and the second best arm has expected reward $1/4$. We construct base algorithms ${\mathcal{B}}_{1}, {\mathcal{B}}_{2}$ as follows. ${\mathcal{B}}_{1}$ always chooses the optimal arm and its expected instantaneous reward is $1/2$. ${\mathcal{B}}_{2}$ chooses the second best arm at time step $t$ with probability $\frac{4c}{\sqrt{t+2}\log(t+2)}$ ($c$ will be specified later), and chooses the best arm otherwise. The expected reward at time step $t$ of ${\mathcal{B}}_{2}$ is $\frac{1}{2}-\frac{c}{\sqrt{t+2}\log(t+2)}$.

Let $A^*$ be uniformly sampled from $\{1,2\}$. Consider two environments $\nu_1$ and $\nu_2$ for the meta-algorithm, each made up of two base algorithms $\widetilde{{\mathcal{B}}}_{1}, \widetilde{{\mathcal{B}}}_{2}$.
 Under $\nu_1$, $\widetilde{{\mathcal{B}}}_{1}$ and $ \widetilde{{\mathcal{B}}}_{2}$ are both instantiations of ${{\mathcal{B}}}_{1}$.
Under $\nu_2$, $\widetilde{{\mathcal{B}}}_{A^*} $, where $A^*$ is a uniformly sampled index in $\{1,2\}$,   is a copy of ${{\mathcal{B}}}_{1}$ and $\widetilde{{\mathcal{B}}}_{3-A^*}$ is a copy of ${{\mathcal{B}}}_{2}$. 

Let $\mathbb{P}_1,\mathbb{P}_2$ denote the probability measures induced by interaction of the meta-algorithm with $\nu_1$ and $\nu_2$ respectively.
Let $\widetilde{{\mathcal{B}}}_{A_t} $ denote the base algorithm chosen by the meta-algorithm at time $t$. We have $\mathbb{P}_1(A_t\neq A^*)=\frac{1}{2}$ for all $t$, since the learner has no information available to identify which algorithm is considered optimal.
By Pinskers' inequality we have
\begin{align*}
    \mathbb{P}_2(A_t\neq A^*) \geq \mathbb{P}_1(A_t\neq A^*)-\sqrt{\frac{1}{2}\mathrm{KL}(\mathbb{P}_1||\mathbb{P}_2)}
\end{align*}
By the divergence decomposition \citep[see][proof of Lemma 15.1 for the decomposition technique]{LS-2020} and using that for $\Delta < \frac{1}{4}:\, \mathrm{KL}(\frac{1}{2},\frac{1}{2}-\Delta) \leq 3\Delta^2$ (Lemma~\ref{lem::lower_bound_helper}), we have
\begin{align*}
    \mathrm{KL}(\mathbb{P}_1||\mathbb{P}_2) &= \sum_{t=2}^\infty \frac{1}{2}\mathrm{KL}\left(\frac{1}{2},\frac{1}{2}-\frac{c}{\sqrt{t+1}\log(t+1)}\right)\\
    &\leq \sum_{t=2}^\infty \frac{3c^2}{2t\log(t)^2} \leq 3c^2\,.
\end{align*}
Picking $c = \sqrt{\frac{1}{24}}$ leads to $\mathbb{P}_2(A_t\neq A^*) \geq \frac{1}{4}$, 
and the regret in environment $\nu_2$ is lower bounded by
\begin{align*}
    R(T) 
    &\geq \sum_{t=1}^T\mathbb{P}_2(A_t\neq A^*)\frac{c}{\sqrt{t+1}\log(t+1)}\\
    &\geq \frac{c}{4\log(T+1)}\sum_{t=1}^T\frac{1}{\sqrt{t+1}} = \Omega\left(\frac{\sqrt{T}}{\log(T)}\right)\,.
\end{align*}
\end{proof}

CORRAL needs knowledge of the best base's regret to achieve the same regret. The following lower bound shows that this requirement is unavoidable: 
\begin{theorem}
\label{thm:lower_bound2}
Let $\mathrm{Alg}$ be a model selection algorithm. There exists a model selection problem with two base algorithms where the best base has regret $\widetilde{O}(T^x)$ for some $0<x<1$ such that if $\mathrm{Alg}$ has no knowledge of $x$ nor of the reward of the best arm, then there exists a potentially different model selection problem where the best base also has regret $\widetilde{O}(T^x)$ but the model selection regret guarantee of $\mathrm{Alg}$ is lower bounded by $\Omega(T^y)$ with $y> x$. 
\end{theorem}

\begin{proof}
Let the set of arms be $\{a_1,a_2,a_3\}$. Let $x$ and $y$ be such that $0< x < y \le 1$. Let $\Delta = T^{x-1 + (y-x)/2}$. Define two environment $\mathcal{E}_1$ and $\mathcal{E}_2$ with reward vectors $\{1,1,0\}$ and $\{1+\Delta,1,0\}$ for $\{a_1,a_2,a_3\}$, respectively. Let $\mathcal{B}_1$ and $\mathcal{B}_2$ be two base algorithms defined by the following fixed policies when running alone in $\mathcal{E}_1$ or $\mathcal{E}_2$: 
\[
\pi_1 = 
\begin{cases}
    a_2       & \quad \text{w.p. } 1-T^{x-1}\\
    a_3  & \quad \text{w.p. } T^{x-1}
\end{cases}
\,,\qquad
\pi_2 = 
\begin{cases}
    a_2       & \quad \text{w.p. } 1-T^{y-1}\\
    a_3  & \quad \text{w.p. } T^{y-1}
\end{cases} \;.
\]
We also construct base $\mathcal{B}'_2$ defined as follows. Let $c_2 > 0$ and $\epsilon_2 = (y-x)/4$ be two constants. Base $\mathcal{B}'_2$ mimics base $\mathcal{B}_2$ when $t\le c_2 T^{x - y+1 + \epsilon_2 }$, and picks arm $a_1$ when $t>  c_2 T^{x - y+1 + \epsilon_2}$. The instantaneous rewards of $B_1$ and $B_2$ when running alone are $r^1_t = 1-T^{x-1}$ and $r^2_t = 1-T^{y-1}$ for all $1 \le t \le T$. Next, consider model selection with base algorithms $\mathcal{B}_1$ and $\mathcal{B}_2$ in $\mathcal{E}_1$. Let $T_1$ and $T_2$ be the number of rounds that $\mathcal{B}_1$ and $\mathcal{B}_2$ are chosen, respectively. 

First, assume case (1): %
There exist constants $c > 0$, $\epsilon > 0$, $p\in (0,1)$, and $T_0>0$ such that with probability at least $p$, $T_2 \ge c T^{x - y+1+ \epsilon}$ for all $T>T_0$.

The regret of base $\mathcal{B}_1$ when running alone for $T$ rounds is $T \cdot T^{x-1} = T^{x}$. The regret of the model selection method is at least 
\[
p\cdot T_2 \cdot T^{y-1} \ge p\cdot c T^{x - y+1+ \epsilon}\cdot T^{y-1}= p\cdot c \cdot T^{x+\epsilon} \;.
\]
Given that the inequality holds for any $T > T_0$, it proves the statement of the lemma in case (1). 

 Next, we assume the complement of case (1): %
For all constants $c > 0$, $\epsilon > 0$, $p\in (0,1)$, and $T_0>0$, with probability at least $p$, $T_2 < c T^{x - y+1+ \epsilon}$ for some $T>T_0$. %

Let $T$ be any such time horizon. Consider model selection with base algorithms $\mathcal{B}_1$ and $\mathcal{B}'_2$ in environment $\mathcal{E}_2$ for $T$ rounds. Let $T'_1$ and $T'_2$ be the number of rounds that $\mathcal{B}_1$ and $\mathcal{B}'_2$ are chosen. Note that $\mathcal{B}_2$ and $\mathcal{B}'_2$ behave the same for $c_2 T^{x - y+1+ \epsilon}$ time steps, and that $\mathcal{B}_1$ and $\mathcal{B}_2$ never choose action $a_1$. Therefore for the first $c_2 T^{x - y+1 + \epsilon_2}$ time steps, the model selection strategy that selects between $\mathcal{B}_1$ and $\mathcal{B}_2'$ in $\mathcal{E}_2$ behaves the same as when it runs $\mathcal{B}_1$ and $\mathcal{B}_2$ in $\mathcal{E}_1$. Therefore with probability $p > 1/2$,  $T'_2 < c_2 T^{x - y+1 + \epsilon_2}$, which implies $T'_1 > T/2$.

In environment $\mathcal{E}_2$, the regret of base $B'_2$ when running alone for $T$ rounds is bounded as
\[
(\Delta +  T^{y-1})  c_2 T^{x - y+1 + \frac{y-x}{4} } =  c_2 T^{\frac{5x-y}{4} } + c_2T^{\frac{3x+y}{4} }< 2c_2 T^{ \frac{3x+y}{4} } 
\]
Given that with probability $p>1/2$, $T'_1 > T/2$, the regret of the learner is lower bounded as,
\[
p (\Delta + T^{x-1}) \cdot \frac{T}{2} >\frac{1}{2}(T^{x-1+\frac{y-x}{2}}+T^{x-1} )\cdot \frac{T}{2}  < \frac{1}{2} T^{\frac{x+y}{2}}\,, 
\]
which is larger than the regret of $\mathcal{B}'_2$ running alone because $ \frac{3x+y}{4} <\frac{x+y}{2}$. The statement of the lemma follows given that for any $T_0$ there exists $T>T_0$ so that the model selection fails.
\end{proof}

\section{Applications of Stochastic CORRAL}\label{sec:applications}

\subsection*{Misspecified Contextual Linear Bandit}

We consider model selection in the misspecified linear bandit problem. The learner selects an action $a_t \in \mathcal{A}_t$ and receives a reward $r_t$ such that $|\E[r_t] - a_t^{\top} \theta| \le  \epsilon_*$ where $\theta \in \mathbb{R}^d$ is an unknown parameter vector and $\epsilon_*$ is the misspecification error. 
For this problem, \citep{zanette2020learning} and \citep{lattimore2019learning} present variants of LinUCB that achieve a high probability $\widetilde{O}(d\sqrt{T} + \epsilon_* \sqrt{d} T)$ regret bound. Both algorithms require knowledge of $\epsilon_*$, but \citep{lattimore2019learning} show a regret bound of the same order without the knowledge of $\epsilon_*$ for the version of the problem with a fixed action set $\mathcal{A}_t=\mathcal{A}$. Their method relies on G-optimal design, which does not work for contextual settings. It is an open question whether it is possible to achieve the above regret without knowing $\epsilon_*$ for problems with changing action sets. 

In this section, we show a $\widetilde{O}(d\sqrt{T} + \epsilon_* \sqrt{d} T)$ regret bound for linear bandit problems with changing action sets without knowing $\epsilon_*$. For problems with fixed action sets, we show an improved regret that matches the lower bound of \citep{LS-2020}.

Given a constant $E$ so that $|\epsilon_*|\le E$, we divide the interval $[1,E]$ into an exponential grid $\mathcal{G}=[1, 2, 2^2,..., 2^{\log(E)}]$. We use $\log(E)$ modified LinUCB bases, from either \citep{zanette2020learning} or \citep{lattimore2019learning}, with each base algorithm instantiated with a value of $\epsilon$ in the grid.  
\begin{theorem}
\label{thm:corral-ucb}
For the misspecified linear bandit problem described above, the regret of Stochastic CORRAL with a CORRAL meta-algorithm using learning rate $\eta = \frac{1}{\sqrt{T}d}$ and LinUCB base algorithms with target misspecification level $\epsilon \in \mathcal{G}$, is upper bounded by $\widetilde{\mathcal{O}}( d \sqrt{T} + \epsilon_* \sqrt{d} T  )$.
In the case of a fixed action linear bandit problem with $k$ arms and $\sqrt{k} >d$, the regret of Stochastic CORRAL with a CORRAL meta-algorithm using learning rate $\eta = \frac{1}{\sqrt{T}d}$ applied to a set of base algorithms consisting of one UCB base and one G-optimal base algorithm \citep{lattimore2019learning} is upper bounded by $\widetilde{\mathcal{O}}\left(\min \left( \frac{k}{d}\sqrt{T}, d\sqrt{T} + \epsilon_* \sqrt{d} T \right) \right)$. 
\end{theorem}

\begin{proof}
 From Lemma~\ref{lem::UCB}, for UCB, $ U(T, \delta) = O(\sqrt{Tk}\log\frac{Tk}{\delta})$. Therefore from Theorem~\ref{thm:meta-algorithm}, running CORRAL with smooth UCB results in the following regret bound: 
\[
\widetilde{O}\left( \sqrt{MT} + \frac{M\ln T}{\eta} + T\eta  + T\left( \sqrt{k}\log\frac{Tk}{\delta} \right)^{2} \eta \right) +\delta T.
\]
If we choose $\delta = 1/T$ and hide some log factors, we get $\widetilde{O}\left( \sqrt{T} +  \frac{1}{\eta} +  T k \eta \right)$. 

For the LinUCB bases in \citep{lattimore2019learning} or \citep{zanette2020learning} or the G-optimal algorithm~\citep{lattimore2019learning}, $U(t,\delta) = O(d\sqrt{t}\log(1/\delta) + \epsilon \sqrt{d} t)$. Substituting $\delta = 1/T$ into Theorem~\ref{thm:meta-algorithm} implies: 
\begin{small}
\begin{align*}
R(T) &\le O\left( \sqrt{MT \log(\frac{4TM}{\delta})}+ \frac{M \ln T}{\eta} + T \eta \right)  -  \EE{\frac{\rho}{40 \eta \ln T} - 2 \rho\, U(T/\rho, \delta) \log T} + \delta T  \\
&\le \widetilde{O}\left( \sqrt{MT } + \frac{M \ln T}{\eta} + T \eta \right)  -  \EE{\frac{\rho}{40 \eta \ln T} - 2 \rho\, (d\sqrt{\frac{T}{\rho}}\log(1/\delta) + \epsilon \sqrt{d} \frac{T}{\rho}) \log T} \\
&\le \widetilde{O}\left( \sqrt{MT } + \frac{M \ln T}{\eta} + T \eta \right)  -  \EE{\frac{\rho}{40 \eta \ln T} - 2 d\sqrt{T\rho}\log(1/\delta) \log T} +  2 \epsilon \, \sqrt{d} T\, \log T
\end{align*}
\end{small}
Maximizing over $\rho$ results in a regret guarantee of the form $\widetilde{\mathcal{O}}\left( \sqrt{T} + \frac{1}{\eta} +  Td^2\eta + \epsilon \sqrt{d} T\right)$. For the misspecified linear bandit problem we use $M = \mathcal{O}(\log(T))$ LinUCB bases with $\epsilon$ defined in the grid, and choose  $\eta = {\frac{1}{\sqrt{T}d}}$. The resulting regret for Stochastic CORRAL is of the form  $\widetilde{\mathcal{O}}\left( \sqrt{T}d  + \epsilon \sqrt{d} T\right)$. 

When the action sets are fixed, by the choice of $\eta = {\frac{1}{\sqrt{T}d}}$, the regret of Stochastic CORRAL with a CORRAL meta-algorithm over one UCB and one G-optimal base equals:
\begin{align*}
&\widetilde{\mathcal{O}}\Bigg(\min\Bigg\{\sqrt{T}\left(d + \frac{k}{d}\right),
\sqrt{T}d + \epsilon \sqrt{d} T \Bigg\}\Bigg) \;.
\end{align*}
If $\sqrt{k} > d$, the above expression becomes $\widetilde{\mathcal{O}}\left( \min \left( \sqrt{T}  \frac{k}{d}  , \sqrt{T} 
d + \epsilon \sqrt{d} T \right) \right)$
\end{proof}

Observe that in the case of a fixed action linear bandit problem, the regret upper bound we achieve for Stochastic CORRAL with a CORRAL meta-algorithm and a learning rate of $\eta = \frac{1}{\sqrt{T}d}$ is of the form $\widetilde{\mathcal{O}}\left(\min \left( \frac{k}{d}\sqrt{T}, d\sqrt{T} + \epsilon_* \sqrt{d} T \right) \right)$. The product of the terms inside the minimum is of order $\widetilde{\mathcal{O}}(kT)$. This result matches the following lower bound that shows that it is impossible to achieve $\widetilde{O}(\min(\sqrt{kT},d\sqrt{T} + \epsilon_* \sqrt{d} T ))$ regret: 
\begin{lemma}[Implied by Theorem 24.4 in \citep{LS-2020}]
\label{lem::lower_bound_ucb} Let $R_{\nu}(T)$ denote the cumulative regret at time $T$ on environment $\nu$.  For any algorithm, there exists a $1$-dimensional linear bandit environment $\nu_1$ and a $k$-armed bandit environment $\nu_2$ such that $R_{\nu_1}(T) \cdot  R_{\nu_2}(T) \ge T(k-1)e^{-2}$. 
\end{lemma}

\paragraph*{Experiment (Figure~~\ref{figcorral:ucb_exp}).} %
Let $d = 2$. Consider a contextual bandit problem with $k = 50$ arms, where each arm $j$ has an associated vector $a_j \in \Real^{d}$ sampled uniformly at random from $[0,1]^d$. We consider two cases: (1) For a $\theta \in \Real^{d}$ sampled uniformly at random from $[0,1]^d$, reward of arm $j$ at time $t$ is $a_j^\top \theta + \eta_t$, where $\eta_t \sim N(0,1)$, and (2) There are $k$ parameters $\mu_j$ for $j\in [k]$ all sampled uniformly at random from $[0,10]$, so that the reward of arm $j$ at time $t$ is sampled from $N(\mu_j, 1)$. We use CORRAL with learning rate $\eta = \frac{2}{\sqrt{T}d}$ and UCB and LinUCB as base algorithm. In case (1) LinUCB performs better while in case (2) UCB performs better. Each experiment is repeated 500 times.
  
\begin{figure}[h]
\begin{minipage}{0.5\textwidth}
\includegraphics[width=\linewidth]{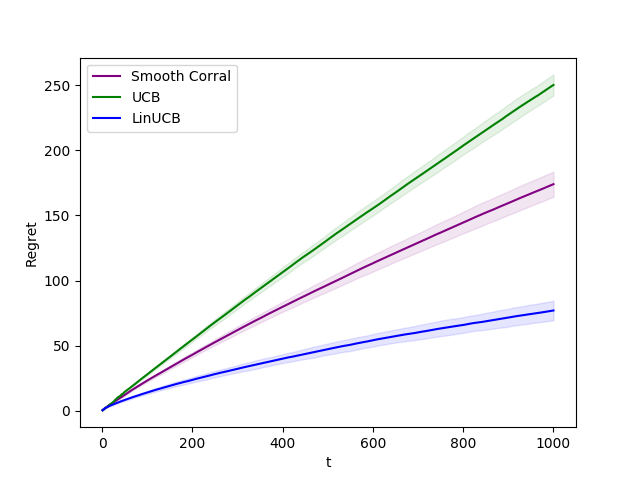}
\begin{center}
Arms with linear rewards.
\end{center}
\end{minipage}
\begin{minipage}{0.5\textwidth}
\includegraphics[width=\linewidth]{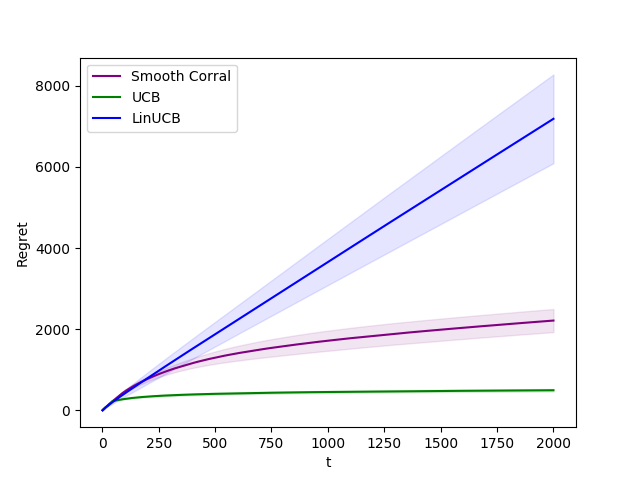}
\begin{center}
Arms with non-linear rewards.
\end{center}
\end{minipage}
\begin{center}
\caption[CORRAL with UCB and LinUCB bases.]{CORRAL with UCB and LinUCB bases. Shaded regions denote the standard deviations.}
    \label{figcorral:ucb_exp}
    \end{center}
\end{figure}

\subsection*{Contextual Bandits with Unknown Dimension}

 We consider model selection in the nested contextual linear bandit problem studied by \citep{foster2019model}. In this problem the context space $\mathcal{A} \subset \mathbb{R}^{D}$. Each action is a $D-$dimensional vector and each context $\mathcal{A}_t$ is a subset of $\mathbb{R}^D$. The unknown parameter vector $\theta_* \in \mathbb{R}^D$ but only its first $d_*$ coordinates are nonzero. Here, $d_*$ is unknown and possibly much smaller than $D$. We assume access to a family of LinUCB algorithms $\{ \mathcal{B}_i \}_{i=1}^M$ with increasing dimensionality $d_i$. Algorithm $i$ is designed to 'believe' the unknown parameter vector $\theta_*$ has only nonzero entries in the first $d_i$ entries.  In \citep{foster2019model} the authors consider the special case when $|\mathcal{A}_t| =  k < \infty$ for all $t$. In order to obtain their model selection guarantees they require a lower bound on the average eigenvalues of the covariance matrices of all actions. In contrast, we do not require any such structural assumptions on the context. We provide the first sublinear regret for this problem when the action set is infinite. Further, we have no eigenvalue assumptions and our regret does not scale with the number of actions $k$.

We use LinUCB with each value of $d \in [1, 2, 2^2,..., 2^{\log(D)}]$  as a base algorithm for CORRAL and EXP3.P.   We also consider the case when both the optimal dimension $d_*$ and the misspecification $\epsilon_*$ are unknown: we use $M=\log(E) \cdot \log(D)$ modified LinUCB bases (see the discussion on Misspecified Contextual Linear Bandits above) for each value of $(\epsilon_*, d_*)$ in the grid $[1, 2, 2^2,..., 2^{\log(E)}] \times [1, 2, 2^2,..., 2^{\log(D)}]$.

From Lemma~\ref{lem::Lin_UCB} and Lemma~\ref{lem::Lin_UCB2}, for linear contextual bandit, LinUCB is $(U, \delta, T)$-bounded with $U(t, \delta) = O(d\sqrt{t}\log(1/\delta))$ for infinite action sets $U-$bounded with $U(t, \delta) = O(\sqrt{dt}\log^3(kT\log(T)/\delta))$ for finite action sets. Choose $\delta = 1/T$ and ignore the log factor, $U(t, \delta) = \widetilde{O}(d\sqrt{t})$ for infinite action sets and $U(t, \delta) = \widetilde{O}(\sqrt{dt})$ for finite action sets. Then $U(t) = c(\delta)t^{\alpha}$ with $\alpha = 1/2$ and $c(\delta) = \widetilde{O}(d)$ for infinite action sets, and $c(\delta) = \widetilde{O}(\sqrt{d})$ for finite action sets. 

Now consider the misspecified linear contextual bandit problem with unknown $d_*$ and $\epsilon_*$. We use the smoothed LinUCB bases \citep{lattimore2019learning, zanette2020learning}. Using the calculation in the proof of Theorem~\ref{thm:corral-ucb} in Section~\ref{sec:applications}, using CORRAL with a smooth LinUCB base with parameters $(d, \epsilon)$ in the grids results in $\widetilde{O}\left( \frac{1}{\eta} +  Td^2\eta + \epsilon \sqrt{d} T\right)$ regret. Since $d$ is unknown, choosing $\eta = 1/\sqrt{T}$ yields the regret $\widetilde{O}\left(  \sqrt{T}d_*^2 + \epsilon \sqrt{d} T\right)$. Using EXP3.P with a smooth LinUCB base with parameters $(d, \epsilon)$ in the grids results in: 
\begin{align*}
 R(T) &= \widetilde{O}\left(\sqrt{MT} + MTp +\frac{1}{p} +  \frac{1}{p} U_i(Tp, \delta) \right) \;. \\
 &= \widetilde{O}\left(\sqrt{MT} + MTp +\frac{1}{p} +  \frac{1}{p} \left(d\sqrt{Tp} + \epsilon \sqrt{d} Tp \right) \right) \;. \\
  &= \widetilde{O}\left(\sqrt{MT} + MTp + \frac{d\sqrt{T}}{p} + \epsilon \sqrt{d} T \right)  \;. 
 \end{align*}
 Since $d_*$ is unknown, choosing $p= T^{-1/3}$ yields a $\widetilde{O}( T^{\frac{2}{3}}{d_*} + \epsilon_*\sqrt{d} T)$ regret bound. We summarize our results in the following table: 
 
 \begin{small}
\begin{center}
\begin{tabular}{ |c|c|c|c| } 
 \hline
 & \multicolumn{2}{c|}{Linear contextual bandit} & \makecell{Misspecified linear \\ contextual bandit} \\ 
 \cline{2-4} 
 & \multicolumn{2}{c|}{Unknown $d_*$} &  \multirow{2}{*}{Unknown $d_*$ and $\epsilon_*$}  \\ 
 \cline{2-3} 
  & Finite action sets & Infinite action sets  &  \\ \hline
 \citep{foster2019model} & \makecell{$\widetilde{O}(T^{2/3}k^{1/3}d_*^{1/3})$ or\\ $\widetilde{O}(k^{1/4}T^{3/4}  + \sqrt{k T d_*  })$}    & N/A & N/A \\ \hline
 EXP3.P & $\widetilde{O}( d_*^{\frac{1}{2}} T^{\frac{2}{3}})$ &  $\widetilde{O}(   d_* T^{\frac{2}{3}}  )$ & $\widetilde{O}( T^{\frac{2}{3}}{d_*} + \epsilon_*\sqrt{d} T)$ \\ \hline
 CORRAL & $\widetilde{O}\left( d_*\sqrt{T } \right)$  &  $\widetilde{O}\left( d_*^2\sqrt{T } \right)$ & $\widetilde{O}\left (\sqrt{T} d_*^2 + \epsilon_*\sqrt{d} T\right)$ \\ 
 \hline
\end{tabular}
\end{center}
\end{small}

\subsection*{ Non-parametric Contextual Bandit.}  

We study model selection in the setting of non-parametric contextual bandits.\citep{AAAI1816944} consider non-parametric stochastic contextual bandits. At time $t$ and given a context $x_t \in \Real^D$, the learner selects arm $a_t \in [k]$ and observes the reward $f(a_t, x_t) + \xi_t$, where $\xi_t$ is a 1-sub-Gaussian random variable and for all $a \in [k]$, the reward function $f(a, \cdot)$ is $L-$lipschitz in the context $x \in \mathbb{R}^D$. It is assumed that the contexts arrive in an IID fashion. \citep{AAAI1816944} obtain a $\widetilde{O}\left(  T^{\frac{1+d}{2+d}} \right)$ regret for this problem. Similar to~\citep{foster2019model}, we assume that only the first $d_*$ context features are relevant for an unknown $d_*< D$. It is important to find $d_*$ because $ T^{\frac{1+d_*}{2+d_*}} \ll T^{\frac{1+D}{2+D}}$. Stochastic CORRAL can successfully adapt to this unknown quantity: we can initialize a smoothed copy of Algorithm $2$ of \citep{AAAI1816944} for each value of $d$ in the grid $[b^0, b^1, b^2,..., b^{\log_b(D)}]$ for some $b > 1$ and  perform model selection with CORRAL and EXP3.P with these base algorithms. %
\begin{small}
\begin{center}
\begin{tabular}{ |c|c|c| } 
 \hline
 & EXP3.P & CORRAL\\ 
\hline
   \makecell{Nonparametric contextual \\bandit with unknown $d_*$}   & $\widetilde{O}\left( T^{\frac{1+bd_*}{2+bd_*} + \frac{1}{3(2+bd_*)}} \right)$ & $\widetilde{O} \left ( T^{\frac{1+2bd_*}{2+2bd_*}}\right)$\\ 
 \hline
\end{tabular}
\end{center}
\end{small}

\subsection*{Tuning the Exploration Rate of $\epsilon$-greedy}

We study the problem of selecting for the optimal scaling for the exploration probability in the $\epsilon$-greedy algorithm. Recall that for a given positive constant $c$, the $\epsilon$-greedy algorithm pulls the arm with the largest empirical average reward with probability $1-c/t$, and otherwise pulls an arm uniformly at random. Let $\epsilon_t=c/t$. It can be shown that the optimal value for $\epsilon_t$ is $\min\{ 1, \frac{5 k }{\Delta_*^2 t}\}$ where $\Delta_*$ is the smallest gap between the optimal arm and the sub-optimal arms~\citep{LS-2020}. With this exploration rate, the regret scales as $\widetilde{ \mathcal{O}}(\sqrt{T})$ for $k=2$. %
 We would like to find the optimal value of $c$ without the knowledge of $\Delta_*$. In this discussion we show it is possible to obtain such result by applying CORRAL to a set of $\epsilon$-greedy base algorithms each instantiated with a $c$ in $[1, 2, 2^2,..., 2^{\log(kT)}]$.

\begin{theorem}
\label{thm:corral-eg}
The regret of CORRAL using smoothed $\epsilon$-greedy base algorithms defined on the grid is bounded by %
$\widetilde{O} (T^{1/2})$ when $k = 2$. 
\end{theorem}

\begin{proof}
 From Lemma~\ref{lem::epsilon_greedy}, we lower bound the smallest gap by $1/T$ (because the gaps smaller than $1/T$ will cause constant regret in $T$ time steps) and choose $\delta = 1/T^5$. From Theorem~\ref{thm:meta-algorithm}, the regret is $\widetilde{O}( T^{2/3})$ when $k> 2$ and $\widetilde{O}( T^{1/2})$ when $k=2$ with the base running alone. 

Next we show that the best value of $c$ in the exponential grid gives a regret that is within a constant factor of the regret above where we known the smallest non-zero gap $\Delta_*$.  An exploration rates can be at most $kT$. Since $\frac{5K}{\Delta_*^2} > 1$, we need to search only in the interval $[1, KT]$. Let $c_1$ be the element in the exponential grid such that $c_1 \leq c^* \leq 2c_1$. Then $2 c_1=\gamma c^*$ where $\gamma<2$ is a constant, and therefore using $2 c_1=\gamma c^*$ will give a regret up to a constant factor of the optimal regret.
\end{proof}

\begin{figure}
\begin{center}
\includegraphics[width = 0.5\textwidth]{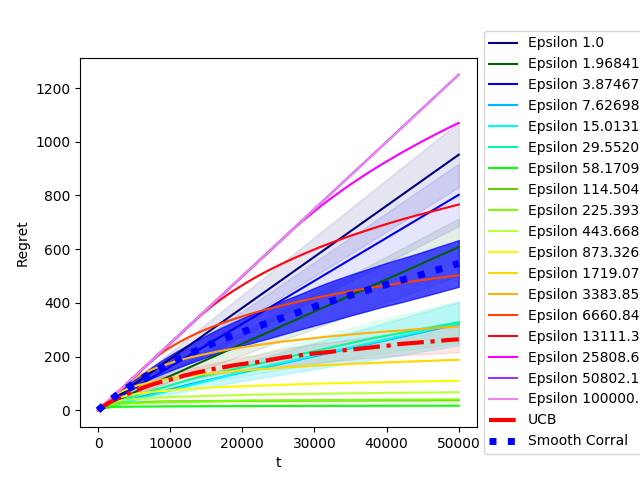} 
\end{center}
\vspace{-1cm}
\caption[CORRAL with $\epsilon$-Greedy bases.]{CORRAL with $\epsilon$-Greedy bases with different exploration rates. \protect\footnotemark }
    \label{figcorral:epsilon_exp}
\end{figure}
\footnotetext{The shaded areas around UCB and CORRAL are the std. The shaded areas around the $\epsilon$-greedy bases are $0.1$ of std. For small $\epsilon$, $\epsilon$-greedy has a very high variance because it either commits to the optimal arm or  the sub-optimal arm at the beginning, so plotting the whole $std$ would make the plot unreadable.}
\paragraph*{Experiment (Figure~\ref{figcorral:epsilon_exp}).} Let there be two Bernoulli arms with means $0.5$ and $0.45$. We use 18 $\epsilon$-greedy base algorithms differing in their choice of $c$ in the exploration rate $\epsilon_t = c/t.$ We take $T=50,000$, $\eta = 20/\sqrt{T}$ and $\epsilon$'s to lie on a geometric grid in $[1,2T].$ Each experiments is repeated $50$ times. %

\subsection*{Reinforcement Learning}

We can instantiate Stochastic CORRAL model selection regret guarantees to the episodic linear MDP setting of \citep{jin2020provably}, again with nested feature classes of doubling dimension just as in the case of the Contextual Bandits with Unknown Dimension. Let's formally define a Linear MDP,
\begin{definition}[Linear MDP (  Assumption A in \citep{jin2020provably})] An episodic MDP (Denoted by the tuple $(\mathcal{S}, A, H, \mathbb{P}, r)$) is a linear MDP with a feature map $\Phi : \mathcal{S} \times A \rightarrow \mathbb{R}^d$, if for any $h \in [H]$ there exist $d$ unknown (signed) measures $\boldsymbol{\mu}_h = (\mu_h^{(1)}, \cdots, \mu_h^{(d)})$ over $\mathcal{S}$ and an unknown vector $\boldsymbol{\theta}_h \in \mathbb{R}^d$, such that for any $(s,a) \in \mathcal{S} \times A$, we have,
\begin{equation*}
    \mathbb{P}_h( \cdot | s, a) = \langle \Phi(s,a) , \boldsymbol{\mu}_h(\cdot) \rangle, \qquad r_h(s,a) = \langle \Phi(s,a), \boldsymbol{\theta}_h\rangle.
\end{equation*}

\end{definition}

The value function for a linear MDP also satisfies a linear parametrization,
\begin{proposition}[Proposition 2.3 from~\citep{jin2020provably}] For a linear MDP, and for any policy $\pi$ there exist $d-$dimensional weights $\{ \mathbf{w}_h^\pi \}_{h\in [H]}$ such that for any $(s,a,h) \in \mathcal{S} \times A \times [H]$ we have that the value function of policy $\pi$ satisfies $Q_h^{\pi}(s,a) = \langle \Phi(s,a), \mathbf{w}_h^{\pi}\rangle$.
\end{proposition}

For the purpose of studying model selection in the setting of linear MDPs we assume access to $D-$dimensional feature maps $\Phi : \mathcal{S} \times A \rightarrow \mathbb{R}^D$. For all policies $\pi$ the unknown parameters $\{ \mathbf{w}_h^\pi \}_{h\in [H]}$ are all assumed to have unknown coordinates only in their first $d_*$ dimensions. We assume access to a family of LSVI-UCB (Algorithm 1 of \citep{jin2020provably}) algorithms $\{ \mathcal{B}_i \}_{i=1}^M$ with increasing dimensionality $d_i$. Algorithm $i$ is designed to `believe' the unknown parameter vectors $\{ \mathbf{w}_h^\pi \}_{h\in [H]}$ has only nonzero entries in the first $d_i$ entries for all policies $\pi$.

\begin{figure}
\begin{center}
\includegraphics[width=0.5\textwidth]{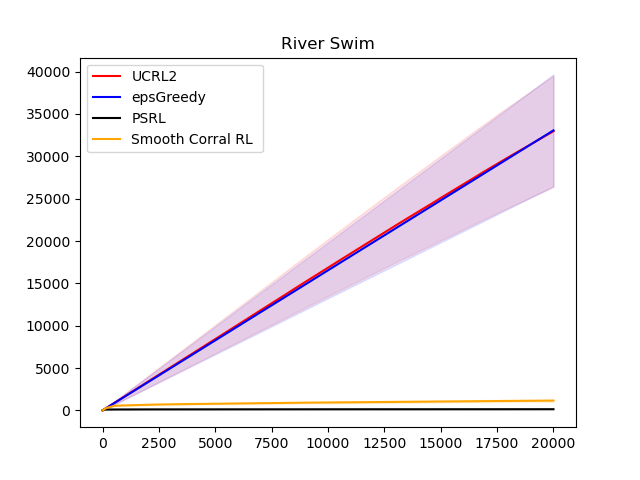} 
\end{center}
\vspace{-1cm}
\caption[$\epsilon$-Greedy vs UCRL2 vs PSRL in the River Swim environment. ]{$\epsilon$-Greedy vs UCRL2 vs PSRL in the River Swim environment \citep{strehl2008analysis}. }%
    \label{figcorral:rl_exp}
\end{figure}

\begin{theorem}
\label{thm:RL}
Let $\mathcal{M}= (\mathcal{S}, A, H, \mathbb{P}, r)$ be a linear MDP parametrized by a feature map $\{ \Phi: \mathcal{S} \times A\rightarrow \mathbb{R}^D\}$.  Let $\{\Phi_i(s,a)\}_{i=1}^M$ be the family of nested feature maps such that $\Phi_i(s,a)$ corresponds to the top $d_i$ entries of $\Phi(s,a)$. Assume that for all policies $\pi$ the unknown parameters $\{ \mathbf{w}_h^{\pi}\}_{h\in[H]}$ have nonzero coordinates only in their first $d_*$ dimensions and that there exists an index $i_*$ such that $d_* \leq d_i \leq 2d_*$. Selecting among different smoothed LSVI-UCB base algorithms corresponding to the feature maps $\{ \Phi_i \}_{i=1}^M$ using Stochastic CORRAL with a CORRAL meta-algorithm and $\eta = \frac{ M^{1/2}}{T^{1/2}d^{3/2}H^{3/2}} $  satisfies a regret guarantee:
    $R(T) \leq \widetilde{\mathcal{O}}\left(\sqrt{M  d^3 H^3 T  } \right)$.
\end{theorem}

\begin{proof}[Proof of Theorem~\ref{thm:RL}] When well specified the LSVI-UCB algorithm~\citep{jin2020provably} satisfies the high probability bound $\widetilde{\mathcal{O}}(\sqrt{d^3H^3T})$ where $H$ is the length of each episode. The result then follows from Theorem~\ref{thm:meta-algorithm} by setting the CORRAL meta-algorithm learning rate as $\eta = \frac{ M^{1/2}}{T^{1/2}d^{3/2}H^{3/2}} $. %
\end{proof}

We also observe that in practice, smoothing RL algorithms such as UCRL and PSRL and using a CORRAL meta-algorithm on top of them can lead to improved performance. In Figure~\ref{figcorral:rl_exp}, we present results for the model selection problem among distinct RL algorithms in the River Swim environment~\citep{strehl2008analysis}. We use three different bases, $\epsilon-$greedy $Q-$learning with $\epsilon=.1$, Posterior Sampling Reinforcement Learning (PSRL), as described in \citep{osband2017posterior} and UCRL2 as described in \citep{jaksch2010near}. The implementation of these algorithms and the environment is taken from TabulaRL (\url{https://github.com/iosband/TabulaRL}), a popular benchmark suite for tabular reinforcement learning problems. Smooth CORRAL uses a CORRAL meta-algorithm with a learning rate $\eta = \frac{15}{\sqrt{T}}$, all base algorithms are smoothed using Algorithm~\ref{Alg:smoothing}. The curves for UCRL2, PSRL and $\epsilon-$greedy are all of their un-smoothed versions. Each experiment was repeated 10 times and we have reported the mean cumulative regret and shaded a region around them corresponding to $\pm .3$ the standard deviation across these 10 runs.

\subsection*{Generalized Linear Bandits with Unknown Link Function}
\citep{li2017provably} study the generalized linear bandit model for the stochastic $k$-armed contextual bandit problem. In round $t$ and given context $x_t \in \mathbb{R}^{d \times k}$, the learner chooses arm $i_t$ and observes reward  $r_t  = \mu(x_{t,i_t}^\top \theta^*) + \xi_t$ where $\theta^* \in\mathbb{R}^d$ is an unknown parameter vector, $\xi_t$ is a conditionally zero-mean random variable and $\mu: \mathbb{R} \rightarrow \mathbb{R}$ is called the link function. \citep{li2017provably} obtain the high probability regret bound $\widetilde{O}(\sqrt{dT})$ where the link function is known. Suppose we have a set of link functions $\mathbb{L}$ that contains the true link function $\mu$. Since the target regret $\widetilde{O}(\sqrt{dT})$ is known, we can run CORRAL with the algorithm in \citep{li2017provably} with each link function in the set as a base algorithm. From Theorem~\ref{thm:meta-algorithm}, CORRAL will achieve regret $\widetilde{O}(\sqrt{|\mathbb{L}|dT})$. 
\subsection*{Bandits with Heavy Tail}
\citep{NIPS2018_8062} study the linear stochastic bandit problem with heavy tail. If the reward distribution has finite moment of order $1+\epsilon_*$, \citep{NIPS2018_8062} obtain the high probability regret bound $\widetilde{O}\left( T^{\frac{1}{1+\epsilon_*}}\right)$. We consider the problem when $\epsilon_* \in (0,1]$ is unknown with a known lower bound $L$ where $L$ is a conservative estimate and $\epsilon_*$ could be much larger than $L$. To the best of our knowledge, we provide the first result when $\epsilon_*$ is unknown. We use the algorithms in \citep{NIPS2018_8062} with value of $\epsilon_*$ in the grid $[ b^{\log_b(L)}, ..., b^1  , b^0]$ for some $0<b<1$ as base algorithms with $\eta = T^{-1/2}$ for CORRAL. A direct application of Theorem~\ref{thm:meta-algorithm} yields regret $\widetilde{O} \left( T^{1-0.5b\epsilon_*}\right)$. When $\epsilon_*=1$ (as in the case of finite variance), $\widetilde{O} \left( T^{1-0.5b\epsilon_*}\right)$ is close to $ \widetilde{O} \left( T^{0.5}\right)$ when $b$ is close to  $1$. 

\section{Conclusion}

In this work we introduced the Stochastic CORRAL algorithm that successfully combines an EXP3 or CORRAL adversarial meta-algorithm with a wide variety of stochastic base algorithms for contextual bandits and reinforcement learning. We improve the results of the original CORRAL approach~\citep{DBLP:conf/colt/AgarwalLNS17} that requires the base algorithms to satisfy a stability condition not often fulfilled by even the simplest stochastic bandit algorithms such as UCB and OFUL. Our approach can make use of the input base algorithms in a fully blackbox fashion without the need of reproving regret bounds for the component base algorithms. This versatility has allowed us to crack several open problems ranging from algorithms that adapt to the misspecification level in linear contextual bandits to the effective dimension in non-parametric problems. 

\printbibliography

\newpage

\appendix
\tableofcontents
\newpage

\section{Omitted proofs of Section~\ref{sec:the_stochastic_corral_algorithm} }\label{app:ommited_proofs_stochastic_corral}

\subsection*{Bounding term  $\mathrm{I}$}

When the base algorithms are not chosen, they repeat their step 2's policy to ensure that the conditional instantaneous regret is decreasing. To ensure the decreasing conditional instantaneous regret serves its purpose, when the base algorithms are chosen by the meta-algorithm, we only send step 2's rewards to the meta-algorithm as feedback signals. This is to ensure that the sequence of rewards the meta-algorithm is competing against satisfies the decreasing instantaneous regret condition. However, since the bases play and incur regrets from both step 1 and step 2 when they are chosen, we must account for the difference between the reward of step 1 and step 2 (that the bases incur when they play the arms), and 2 times the reward of step 2 (what the bases send to the meta-algorithm as feedback signals). 

Since we assume all base algorithms to be smoothed and satisfy a two step feedback structure, we also denote by $\pi_t^{(j)}$ as the policy used by the meta-algorithm during round $t$, step $j$. Term I, the regret of the meta-algorithm with respect to base $i_\star$ can be written as:
\begin{align}
    \EE{\mathrm{I}} = \EE{\sum_{t=1}^T \sum_{j=1}^2 f(\mathcal{A}_t^{(j)}, \pi_{t,i_\star}^{(j)}) - f(\mathcal{A}_t^{(j)}, \pi_t^{(j)})}
\end{align}
The reader should keep in mind the meta-algorithm is updated only using the reward of Step 2 of base algorithms even though the bases play both step 1 and 2. Let $\mathbb{T}_i$ be the random subset of rounds when $\mathcal{M}$ choose base $\widetilde{\mathcal{B}}_i$, ($i_t = i$) for all $i \in [M]$. Adding and subtracting terms $\{ f(\mathcal{A}_t^{(1)}, \pi_t^{(2)} ) \}_{t=1}^T$ we see that:

\begin{align*}
    \mathrm{I} &= \sum_{t=1}^T \sum_{j=1}^2 f(\mathcal{A}_t^{(j)}, \pi_{t,i_\star}^{(j)}) - f(\mathcal{A}_t^{(j)}, \pi_t^{(j)}) \\
    &= \underbrace{\sum_{t\in \mathbb{T}_{i_\star}}^T \sum_{j=1}^2 f(\mathcal{A}_t^{(j)}, \pi_{t,i_\star}^{(j)}) - f(\mathcal{A}_t^{(j)}, \pi_t^{(j)})}_{\mathrm{I}_0} + \underbrace{\sum_{t\in \mathbb{T}_{i_\star}^c} \sum_{j=1}^2 f(\mathcal{A}_t^{(j)}, \pi_{t,i_\star}^{(j)}) - f(\mathcal{A}_t^{(j)}, \pi_t^{(j)})}_{\mathrm{I}_1}\\ 
    &\stackrel{(i)}{=} \underbrace{\sum_{t\in \mathbb{T}_{i_\star}}^T \sum_{j=1}^2 f(\mathcal{A}_t^{(j)}, \pi_{t,i_\star}^{(2)}) - f(\mathcal{A}_t^{(j)}, \pi_t^{(2)})}_{\mathrm{I}_0'} + \underbrace{\sum_{t\in \mathbb{T}_{i_\star}^c} \sum_{j=1}^2 f(\mathcal{A}_t^{(j)}, \pi_{t,i_\star}^{(2)}) - f(\mathcal{A}_t^{(j)}, \pi_t^{(j)})}_{\mathrm{I}_1'}\\
    &\stackrel{(ii)}{=}  \underbrace{\sum_{t=1}^T \sum_{j=1}^2 f(\mathcal{A}_t^{(j)}, \pi_{t,i_\star}^{(2)}) - f(\mathcal{A}_t^{(j)}, \pi_t^{(2)})}_{\mathrm{I}_A} + \underbrace{\sum_{t\in \mathbb{T}^c_{i_\star}}  f(\mathcal{A}_t^{(1)}, \pi_{t}^{(2)}) - f(\mathcal{A}_t^{(1)}, \pi_t^{(1)})}_{\mathrm{I}_B}
    \end{align*}
    
Equality $(i)$ holds because term $\mathrm{I}_0$ equals zero (recall for all $t \in \mathbb{T}_{i_\star}$ algorithm $i_\star$ is chosen by the meta-algorithm) and therefore $\mathrm{I}_0 = \mathrm{I}_0'$ and in all steps $t \in \mathbb{T}_{i_\star}^c$, base $i_\star$ repeated a policy of Step $2$ so that $\mathrm{I}_1 = \mathrm{I}_1'$.  Equality $(ii)$ follows by adding and subtracting term $\mathrm{I}_B$. Term $\EE{\mathrm{I}_A}$ is the regret of the meta-algorithm with respect to base $i_\star$. Term $\EE{\mathrm{I}_B}$ accounts for the difference between the rewards of Step 1 and Step 2 (that the bases incur) and 2 times the rewards of Step 2 (that the bases send to the meta-algorithm).
We now focus on bounding 
$\EE{\mathrm{I}_A}$ and $\EE{\mathrm{I}_B}$.

\paragraph*{Biased step $2$'s rewards. } We set the bias functions to $b_j(s) =\frac{U(s, \delta)}{s}$ in Algorithm~\ref{Alg:smooth_meta-algorithm}. This will become useful to control $\EE{\mathrm{I}_B}$. Instead of sending the meta-algorithm the unadulterated $2r_{t,j}^{(2)}$ feedback, at all time step $t$, all bases will send the following modified feedback:

\begin{equation}\label{equation::modified_step_2_rewards}
\widetilde{r}_{t,j}^{(2)}= r_{t,j}^{(2)} - \underbrace{\frac{U_j(s_{t,j}, \delta)}{s_{t,j}}}_{b_j(s_{t,j})}
\end{equation}
\noindent This reward satisfies:
\begin{align*}
    \EE{ \widetilde{r}_{t,j}^{(2)} |\mathcal{F}_{t-1} }&= \EE{f(\mathcal{A}_t^{(2)}, \pi_t^{(2)}) | \mathcal{F}_{t-1}} -\frac{ U_j(s_{t,j}, \delta)}{s_{t,j}}
    \end{align*} 
Define the modified rewards $\widetilde{f}(\mathcal{A}, \pi_{t, j}^{(2)}) = f(\mathcal{A}, \pi_{t,j}^{(2)})  - b_j(s_{t,j})$ for all $j \in [M]$ and context $\mathcal{A}$. Let's write $\mathrm{I}_A + \mathrm{I}_B$ in terms of these $\widetilde{f}$.

\begin{align}
\mathrm{I}_A +\mathrm{I}_B &=   \underbrace{\sum_{t=1}^T \sum_{j=1}^2 \widetilde{f}(\mathcal{A}_t^{(j)}, \pi_{t,i_\star}^{(2)}) - \widetilde{f}(\mathcal{A}_t^{(j)}, \pi_t^{(2)})}_{\widetilde{\mathrm{I}}_A} + \underbrace{\sum_{t\in \mathbb{T}^c_{i_\star}}  \widetilde{f}(\mathcal{A}_t^{(1)}, \pi_{t}^{(2)}) - f(\mathcal{A}_t^{(1)}, \pi_t^{(1)})}_{\widetilde{\mathrm{I}}_B} + \notag \\
&\quad  \sum_{t=1}^T  2 b_{i_\star}(s_{t,i_\star}) -  2 b_{j_t}(s_{t, j_t})  + \sum_{t\in \mathbb{T}^c_{i_\star}} b_{j_t}(s_{t,j_t})  \notag  \\
&\stackrel{(i)}{\leq} \widetilde{\mathrm{I}}_A + \widetilde{\mathrm{I}}_B + \sum_{t=1}^T 2b_{i_\star}(s_{t,i_\star}) \label{equation::upper_bound_IA_IB_bs}
\end{align}

Where inequality $(i)$ holds because $\sum_{t=1}^T 2b(s_{t,j_t}) - \sum_{t \in \mathbb{T}_{i_\star}^c} b(s_{t, j_t}) \leq 0$. In the coming discussion we'll show that this modification allows us to control term $\widetilde{\mathrm{I}}_B$. In the following two sections we will control $\mathbb{E}\left[ \widetilde{\mathrm{I}}_A \right]$ and $\mathbb{E}\left[  \widetilde{\mathrm{I}}_B  \right]$. We will control $\mathbb{E}\left[ \widetilde{\mathrm{I}}_A \right]$ by using standard arguments from the adversarial bandits literature. We will also show that with high probability $\mathbb{E}\left[  \widetilde{\mathrm{I}}_B  \right] \leq 8 \sqrt{MT \log(\frac{4TM}{\delta}) }$. The use of the biased rewards $\widetilde{r}_{t,j}^{(2)}$ allows us to ensure the collected reward during steps of type $1$ plus the bias terms vs the collected reward of steps of type $2$ is close to zero. Without these bias terms, bounding term $\mathrm{I}_B$ may prove problematic since the rewards of steps of type $1$ may be smaller than the collected rewards of steps of type $2$. In this case $\mathbb{E}[\mathrm{I}_B]$ may give rise to a regret term dependent on the putative regret upper bounds of all algorithms $j \in [M]$ and not only on $U_{\star}(T, \delta)$.

\subsection*{Bounding term $\EE{\widetilde{\mathrm{I}}_A}$}

Let's start by noting that after taking expectations,

\begin{align}\label{equation::term_tilde_I_A}
    \mathbb{E}\left[ \widetilde{\mathrm{I}}_A \right] = 2 \mathbb{E}\left[ \sum_{t=1}^T \widetilde{f}(\mathcal{A}_t^{(2)}, \pi_{t,i_\star}^{(2)}) - \widetilde{f}(\mathcal{A}_t^{(2)}, \pi_t^{(2)})  \right]  
\end{align}

The modification of the bases' rewards in Equation~\ref{equation::modified_step_2_rewards} modifies both the bases rewards as well as the comparator.  Since both meta-algorithms CORRAL and EXP3.P are $k$-armed bandit adversarial algorithms, their worst-case performance guarantees hold for this biased pseudo-reward sequence. 

\subsubsection*{CORRAL Meta-Algorithm}

We can bound Equation~\ref{equation::term_tilde_I_A} using Lemma 13 from \citep{DBLP:conf/colt/AgarwalLNS17}. Indeed, in term $\widetilde{\mathrm{I}}_A$, the policy choice for all base algorithms $\{\widetilde{\mathcal{B}}_m\}_{m =1}^M$ during any round $t$ is chosen before the value of $i_t$ is revealed. This ensures the estimates $\frac{2r_t^{(2)}}{p_t^{i_t}}$ and $0$ for all $i\neq i_t$ are indeed unbiased estimators of the base algorithm's rewards. We conclude:
\begin{equation*}
    \EE{\mathrm{I}_A} \leq \mathcal{O}\left( \frac{M \ln T}{\eta} + T \eta \right)- \frac{\EE{\frac{1}{\underline{p}_{i_\star}}}}{40 \eta \ln T} 
\end{equation*}
\subsubsection*{EXP3.P Meta-Algorithm}
Since $\EE{\mathrm{I}_A}$ is the regret of base $i$ with respect to the meta-algorithm, it can be upper bounded by the $k$-armed bandit regret of the meta-algorithm with $M$ arms. Choose $\eta = 1, \gamma = 2k\beta$ in Theorem 3.3 in \citep{Bubeck-Slivkins-2012}, we have that if $p \le \frac{1}{2k}$, the regret of EXP3.P: 
\begin{align*}
    \EE{\mathrm{I}_A} \le \widetilde{\mathcal{O}} \left(MTp + \frac{\log(k\delta^{-1})}{p}\right)
\end{align*}
\subsection*{Bounding $\EE{\widetilde{\mathrm{I}}_B}$}
Notice that:
\begin{align*}
    \EE{\widetilde{\mathrm{I}}_B} &= \EE{\sum_{t\in \mathbb{T}_{i_\star}^c}  \widetilde{f}(\mathcal{A}_t^{(1)}, \pi_{t}^{(2)}) - f(\mathcal{A}_t^{(1)}, \pi_t^{(1)})}\\
    &=\EE{\sum_{t\in \mathbb{T}_{i_\star}^c}  \widetilde{f}(\mathcal{A}_t^{(2)}, \pi_{t}^{(2)}) - f(\mathcal{A}_t^{(1)}, \pi_t^{(1)})}\\
    &= \EE{\sum_{t\in \mathbb{T}_{i_\star}^c}  \widetilde{f}(\mathcal{A}_t^{(2)}, \pi_{t}^{(2)}) - f(\mathcal{A}_t^{(2)}, \pi^*) + f(\mathcal{A}_t^{(2)}, \pi^*) - f(\mathcal{A}_t^{(1)}, \pi_t^{(1)})}\\
    &= \EE{\sum_{t\in \mathbb{T}_{i_\star}^c}  \widetilde{f}(\mathcal{A}_t^{(2)}, \pi_{t}^{(2)}) - f(\mathcal{A}_t^{(2)}, \pi^*) + f(\mathcal{A}_t^{(1)}, \pi^*) - f(\mathcal{A}_t^{(1)}, \pi_t^{(1)})}
\end{align*}

Substituting the definition of $\widetilde{f}( \mathcal{A}_t^{(2)}, \pi_t^{(2)})$ and $b_j(s_{t,j})$ back into the expectation for $\EE{\widetilde{\mathrm{I}}_B}$ becomes:
\begin{center}
\begin{align}
     \mathbb{E}[\widetilde{\mathrm{I}}_B] &= \mathbb{E}\Big[ \sum_{t\in \mathbb{T}_{i_\star}^c}  f(\mathcal{A}_t^{(2)}, \pi_{t}^{(2)}) - f(\mathcal{A}_t^{(2)}, \pi^*) - \frac{U_{j_t}(s_{t,j_t}, \delta)}{s_{t,j_t}} +  f(\mathcal{A}_t^{(1)}, \pi^*) - f(\mathcal{A}_t^{(1)}, \pi_t^{(1)})\Big] \notag\\
     &\stackrel{(1)}{=} \sum_{j \neq i_\star}\mathbb{E}\Big[ \sum_{t \in \mathbb{T}_j } f(\mathcal{A}_t^{(2)}, \pi_{t,j}^{(2)}) - f(\mathcal{A}_t^{(2)}, \pi^*) - \frac{U_{j}(s_{t,j}, \delta)}{s_{t, j}} +  f(\mathcal{A}_t^{(1)}, \pi^*) - f(\mathcal{A}_t^{(1)}, \pi_{t,j}^{(1)}) \Big]  \notag \\
     &\stackrel{(2)}{\leq}  \sum_{j \neq i_\star}\mathbb{E}\Big[ \sum_{t \in \mathbb{T}_j } f(\mathcal{A}_t^{(2)}, \pi_{t,j}^{(2)}) - f(\mathcal{A}_t^{(2)}, \pi^*)  + f(\mathcal{A}_t^{(1)}, \pi^*) - f(\mathcal{A}_t^{(1)}, \pi_{t,j}^{(1)})  \Big] - U_{j}(s_{T, j}, \delta)
     \label{equation::first_bound_I_B}
\end{align}
\end{center}
Equality $(1)$ follows by noting $\mathbb{T}_{i_\star}^c = \cup_{j \neq i_\star} \mathbb{T}_j$. Inequality $(2)$ follows because by Lemma \ref{lem::helper}, we have $ U_j(s_{T, j}, \delta) \leq \sum_{s=1}^{s_{T,j}} \frac{U(s, \delta)}{s}$ for all $j \in [M]$.  If the $j-$th algorithm was adapted to the environment, then with high probability satisfies the following bound:
\begin{small}
\begin{align}
    \left(\sum_{t \in \mathbb{T}_j } f(\mathcal{A}_t^{(2)}, \pi_{t, j}^{(2)}) - f(\mathcal{A}_t^{(2)}, \pi^*) + f(\mathcal{A}_t^{(1)}, \pi^*) - f(\mathcal{A}_t^{(1)}, \pi_{t, j}^{(1)}) \right) - U_j(T_j(T), \delta) &\stackrel{(A)}{\leq} \notag\\ \left(\sum_{t \in \mathbb{T}_j }  f(\mathcal{A}_t^{(1)}, \pi^*) - f(\mathcal{A}_t^{(1)}, \pi_{t, j}^{(1)}) \right) - U_j(T_j(T), \delta) &\stackrel{(B)}{\leq} 0 \label{equation::non_positive}
\end{align}
\end{small}
Inequality $(A)$ follows because by definition $f(\mathcal{A}_t^{(2)}, \pi^*) \geq f(\mathcal{A}_t^{(2)}, \pi_t^{(2)})$ and $(B)$ because if $\mathcal{B}_j$ is adapted to the environment it satisfies a high probability regret bound. Let $\mathrm{Adapt} = \{ j \in [M] \text{ s.t. } j \text{ is adapted } \} $. Let's rewrite the upper bound for $\mathbb{E}\left[ \widetilde{\mathrm{I}}_B \right]$ from Equation~\ref{equation::first_bound_I_B} as a sum of terms corresponding to base algorithms $j \in \mathrm{Adapt}$ and $j \in [M]\backslash \mathrm{Adapt}$.

\begin{small}
\begin{align*}
     \mathbb{E}[\widetilde{\mathrm{I}}_B] &\leq \sum_{j \neq i_\star}\mathbb{E}\Big[ \sum_{t \in \mathbb{T}_j } f(\mathcal{A}_t^{(2)}, \pi_{t,j}^{(2)}) - f(\mathcal{A}_t^{(2)}, \pi^*)  + f(\mathcal{A}_t^{(1)}, \pi^*) - f(\mathcal{A}_t^{(1)}, \pi_{t,j}^{(1)})  \Big] - U_{j}(s_{T, j}, \delta) \\
     &= \sum_{j \neq i_\star, j \in \mathrm{Adapt}}\mathbb{E}\Big[ \sum_{t \in \mathbb{T}_j } f(\mathcal{A}_t^{(2)}, \pi_{t,j}^{(2)}) - f(\mathcal{A}_t^{(2)}, \pi^*)  + f(\mathcal{A}_t^{(1)}, \pi^*) - f(\mathcal{A}_t^{(1)}, \pi_{t,j}^{(1)})  \Big] - U_{j}(s_{T, j}, \delta) +  \\
     & \sum_{j \neq i_\star, j \not\in \mathrm{Adapt}}\mathbb{E}\Big[ \sum_{t \in \mathbb{T}_j } f(\mathcal{A}_t^{(2)}, \pi_{t,j}^{(2)}) - f(\mathcal{A}_t^{(2)}, \pi^*)  + f(\mathcal{A}_t^{(1)}, \pi^*) - f(\mathcal{A}_t^{(1)}, \pi_{t,j}^{(1)})  \Big] - U_{j}(s_{T, j}, \delta) 
\end{align*}
\end{small}

Equation~\ref{equation::non_positive} implies 
that with probability at least $1-\left| \mathrm{Adapt}\right|\delta$, 

\begin{equation*}
    \sum_{j \neq i_\star, j \in \mathrm{Adapt}}\mathbb{E}\Big[ \sum_{t \in \mathbb{T}_j } f(\mathcal{A}_t^{(2)}, \pi_{t,j}^{(2)}) - f(\mathcal{A}_t^{(2)}, \pi^*)  + f(\mathcal{A}_t^{(1)}, \pi^*) - f(\mathcal{A}_t^{(1)}, \pi_{t,j}^{(1)})  \Big] - U_{j}(s_{T, j}, \delta) \leq 0.
\end{equation*}

We are left with controlling the component of the upper bound of $\mathbb{E}\left[ \widetilde{\mathrm{I}}_B \right]$ that runs over misspecified algorithms. When $\mathcal{B}_j$ is not adapted, Equation~\ref{equation::non_positive} may or may not hold. In order to ensure we are able to control $\mathbb{E}\left[ \widetilde{\mathrm{I}}_B\right]$ we will make sure that algorithms that violate Equation~\ref{equation::non_positive} by a large margin are dropped by the meta-algorithm. Since it is impossible to compute the terms $f(\mathcal{A}_t^{(2)}, \pi_t^{(2)}) - f(\mathcal{A}_t^{(2)}, \pi^*)$ and $f(\mathcal{A}_t^{(1)}, \pi^*) - f(\mathcal{A}_t^{(1)}, \pi_t^{(1)})$ directly, we instead rely on the following test:

\paragraph*{Base Test.} Let $\mathbb{T}_j(l)$ be the set of time indices in $[l]$ when the meta-algorithm chose to play base $j$. We drop base $\widetilde{\mathcal{B}}_j$ if at any point during the history of the algorithm, 
\begin{equation}
    \sum_{t\in \mathbb{T}_j(l) }  r_{t, j}^{(2)} - r_{t, j}^{(1)} > U_j(T_j(\ell), \delta) + 2 \sqrt{2 \left| T_j(\ell)\right| \log\left( \frac{4TM}{\delta}\right) } \label{equation::natural_environment_condition}
\end{equation}

Let's start by showing that with high probability $ \sum_{t\in \mathbb{T}_j(l) }  r_{t, j}^{(2)} - r_{t, j}^{(1)}$ is a good estimator of $\sum_{t \in \mathbb{T}_j(l) } f(\mathcal{A}_t^{(2)}, \pi_{t, j}^{(2)}) - f(\mathcal{A}_t^{(2)}, \pi^*) + f(\mathcal{A}_t^{(1)}, \pi^*) - f(\mathcal{A}_t^{(1)}, \pi_{t, j}^{(1)}) $ for all $j \in [M]$.

As a simple consequence of the Azuma-Hoeffding martingale bound and Assumption~\ref{assumption:unit_bounded_reward}, with probability at least $1-\delta/M$ and for all $\ell \in [T]$ and for any $j \in [M]$:

\begin{align}
   \left| \sum_{t \in \mathbb{T}_j(\ell)} f(\mathcal{A}_t^{(2)}, \pi^*) - f(\mathcal{A}_t^{(1)}, \pi^*) \right|&\leq \sqrt{2 \left|  \mathbb{T}_j(\ell)  \right|  \log\left(\frac{ 4TM}{\delta}\right) } \label{equation::margingale_helper_bound1} \\
   \left| \sum_{t \in \mathbb{T}_j(\ell)} r_{t, j}^{(2)} - r_{t, j}^{(1)} - f( \mathcal{A}_t^{(2)}, \pi_{t, j}^{(2)}) - f(\mathcal{A}_t^{(1)}, \pi_{t, j}^{(1)})\right| &\leq \sqrt{2 \left|  \mathbb{T}_j(\ell)  \right| \log\left( \frac{4TM}{\delta}\right) } \label{equation::martingale_helper_bound2}
\end{align}
Combining Equation~\ref{equation::margingale_helper_bound1} and Equation~\ref{equation::martingale_helper_bound2} we get, with probability at least $1-\frac{\delta}{M}$ for all $l \in [T]$ and for any $j \in [M]$:
\begin{small}
\begin{align}
\left| \left(\sum_{t\in \mathbb{T}_j(l) }  r_{t, j}^{(2)} - r_{t, j}^{(1)}     \right)  - \left(\sum_{t \in \mathbb{T}_j(l) } f(\mathcal{A}_t^{(2)}, \pi_{t, j}^{(2)}) - f(\mathcal{A}_t^{(2)}, \pi^*) + f(\mathcal{A}_t^{(1)}, \pi^*) - f(\mathcal{A}_t^{(1)}, \pi_{t, j}^{(1)}) \right) \right| \notag \\
\leq 2 \sqrt{2 \left|  \mathbb{T}_j(\ell)  \right| \log\left( \frac{4TM}{\delta}\right) } \label{equation::r_2_r_1_proxy}
\end{align}
\end{small}

Thus $ \sum_{t\in \mathbb{T}_j(l) }  r_{t, j}^{(2)} - r_{t, j}^{(1)}$  is a good proxy for $\sum_{t \in \mathbb{T}_j(l) } f(\mathcal{A}_t^{(2)}, \pi_{t, j}^{(2)}) - f(\mathcal{A}_t^{(2)}, \pi^*) + f(\mathcal{A}_t^{(1)}, \pi^*) - f(\mathcal{A}_t^{(1)}, \pi_{t, j}^{(1)})$. Let's start by noting that in case $\widetilde{\mathcal{B}}_j$ is adapted to the environment, the test in Equation~\ref{equation::natural_environment_condition} will not trigger. In particular the optimal algorithm $i_\star$ will not be eliminated.

Equation~\ref{equation::non_positive} holds for all $j \in  \mathrm{Adapt}$ with probability at least $1-\left|\mathrm{Adapt}\right|\delta$. Combining this result with Equation~\ref{equation::r_2_r_1_proxy} we conclude that with probability at least $1-|\mathrm{Adapt}|(1+\frac{1}{M})\delta$ for all $j \in \mathrm{Adapt}$ and all $\ell \in [T]$,
\begin{equation*}
    \sum_{t\in \mathbb{T}_j(l) }  r_{t, j}^{(2)} - r_{t, j}^{(1)} \leq U_j(T_j(\ell), \delta) + 2 \sqrt{2 \left| T_j(\ell)\right| \log\left( \frac{4TM}{\delta}\right) }
\end{equation*}
for all $\ell \in [T]$. Thus with high probability no well adapted algorithm will be eliminated.

Let's now show that for all $j \in [M]\backslash \mathrm{Adapt}$ the contribution of $\widetilde{\mathcal{B}}_j$ to $\mathbb{E}\left[ \widetilde{\mathrm{I}}_B\right]$ while the test of Equation~\ref{equation::natural_environment_condition} has not been triggered is small. If Equation~\ref{equation::natural_environment_condition} holds for algorithm $j \in [M]$ (even if $\widetilde{\mathcal{B}}_j$ is not adapted), then Equation~\ref{equation::r_2_r_1_proxy} implies that with probability at least $1-\frac{\delta}{M}$:
\begin{small}
\begin{align}
\sum_{t \in \mathbb{T}_j(l)}    f(\mathcal{A}_t^{(2)}, \pi_{t, j}^{(2)})  - f(\mathcal{A}_t^{(2)}, \pi^*) + f(\mathcal{A}_t^{(1)}, \pi^*)-  f(\mathcal{A}_t^{(1)}, \pi_{t, j}^{(1)}) - U_j(s_{l,j}, \delta) \leq
 4 \sqrt{2 | \mathbb{T}_j(l)|\log\left(  \frac{4TM}{\delta}\right) }\label{equation::maintaining_boundedness_condition}
\end{align}
\end{small}
This test guarantees that with probability at least $1-\frac{\left|[M]\backslash \mathrm{Adapt}\right|\delta}{M}$,

\begin{align*}
    \sum_{j \neq i_\star, j \not\in \mathrm{Adapt}}\mathbb{E}\Big[ \sum_{t \in \mathbb{T}_j } f(\mathcal{A}_t^{(2)}, \pi_{t,j}^{(2)}) - f(\mathcal{A}_t^{(2)}, \pi^*)  + f(\mathcal{A}_t^{(1)}, \pi^*) - f(\mathcal{A}_t^{(1)}, \pi_{t,j}^{(1)})  \Big] - U_{j}(s_{T, j}, \delta) &\leq \\
    \sum_{j \neq i_\star, j \in [M]\backslash \mathrm{Adapt}}  4 \sqrt{2 | \mathbb{T}_j|\log\left(  \frac{4TM}{\delta}\right)}&\leq \\
     8 \sqrt{ \left|[M]\backslash \mathrm{Adapt}\right| T \log\left( \frac{4TM}{\delta}\right)  }
\end{align*}
The last inequality holds because $\sum_{j \neq i_\star} \sqrt{ | \mathbb{T}_j |} \leq \sqrt{TM}$. And therefore,

\begin{equation*}
    \EE{\widetilde{\mathrm{I}}_B} \leq  8 \sqrt{\left|[M]\backslash \mathrm{Adapt}\right|T \log\left(\frac{4TM}{\delta}\right) } + T(M+1)\delta
\end{equation*}

\subsection*{Bounding term $\mathrm{II}$}

Recall term $\mathrm{II}$ equals:

\begin{equation}
  \EE{\mathrm{II}} =     \EE{\sum_{t=1}^T   f(\mathcal{A}_t, \pi^*) - f(\mathcal{A}_t, \pi_{s_{t,i},i})}
\end{equation}

We use $n_t^i$ to denote the number of rounds base $i$ is chosen up to time $t$ for all $i \in [M]$. Let $t_{l,i}$ be the round index of the $l-$th time the meta-algorithm chooses algorithm $\mathcal{B}_i$ and let $b_{l,i} = t_{l,i} - t_{l-1,i}$ with $t_{0,i} = 0$ and $t_{n_T^i + 1,i} = T+1$. 
Let $\mathbb{T}_i \subset [T]$ be the set of rounds where base $i$ is chosen and $\mathbb{T}_i^c = [T]\backslash \mathbb{T}_i$. For $S \subset [T]$ and $j\in \{1,2\}$, we define the regret of the $i-$th base algorithm during Step $j$ of rounds $S$ as $R_i^{(j)}(S) = \sum_{t\in S} f(\mathcal{A}_t^{(j)}, \pi^*) -f(\mathcal{A}_t^{(j)}, \pi_{t,i}^{(j)}) $. The following decomposition of $\EE{\mathrm{II}}$ holds:
\begin{equation}\label{equation::decomposition_term_II}
    \EE{\mathrm{II}} = \E\left[  R_{i_\star}^{(1)}(\mathbb{T}_{i_\star}) + \underbrace{ R_{i_\star}^{(2)}(\mathbb{T}_{i_\star})+ R_{i_\star}^{(1)}(\mathbb{T}_{i_\star}^c)  +  R_{i_\star}^{(2)}(\mathbb{T}_{i_\star}^c)}_{\mathrm{II}_0} \right]. 
\end{equation}
$R_{i_\star}^{(1)}(\mathbb{T}_{i_\star})$ consists of the regret when base $i_\star$ was updated in step 1 while the remaining $3$ terms consists of the regret when the policies are reused by step 2. 

\subsection*{Biased step $2$'s rewards}

Note that we modified the rewards of step 2 as defined in Equation~\ref{equation::modified_step_2_rewards}, both when the base is chosen and not chosen. We now analyze the effect of this modification: 
\begin{small}
\begin{align*}
&R(T)\\&= \EE{\sum_{t=1}^T \sum_{j=1}^2 f(\mathcal{A}_t^{(j)}, \pi^*) - f(\mathcal{A}^{(j)}_t, \pi^{(j)}_t) } \notag \\ 
&=\mathbb{E}\underbrace{ \left[\sum_{t=1}^T \sum_{j=1}^2 f(\mathcal{A}^{(j)}_t, \pi^{(j)}_{s_{t,i_\star},i_\star}) - f(\mathcal{A}^{(j)}_t, \pi^{(j)}_t)\right]}_{\mathrm{I}}+
 \mathbb{E}\underbrace{\left[\sum_{t=1}^T \sum_{j=1}^2  f(\mathcal{A}^{(j)}_t, \pi^*) - f(\mathcal{A}^{(j)}_t, \pi^{(j)}_{s_{t,i_\star},i_\star})\right]}_{\mathrm{II}} \\ 
&=\mathbb{E} \left[\sum_{t=1}^T \sum_{j=1}^2 \left( f(\mathcal{A}^{(j)}_t, \pi^{(j)}_{s_{t,i_\star},i_\star}) - \mathbf{1}( t \in \mathbb{T}_{i_\star}^c \text{ or }  j = 2) \frac{U_i(s_{t,i_\star}, \delta)}{s_{t,i_\star}}\right) -  f(\mathcal{A}^{(j)}_t, \pi^{(j)}_t) \right]\\
&\quad
 +\mathbb{E}\left[\sum_{t=1}^T \sum_{j=1}^2  f(\mathcal{A}^{(j)}_t, \pi^*) - \left( f(\mathcal{A}^{(j)}_t, \pi^{(j)}_{s_{t,i_\star},i_\star}) -\mathbf{1}( t \in \mathbb{T}_{i_\star}^c \text{ or }  j = 2) \frac{U_{i_\star}(s_{t,i_\star}, \delta)}{s_{t,i_\star}}\right)\right] \\
 &\leq \mathrm{I}-\text{modified} + \mathrm{II}-\text{modified}
\end{align*}
\end{small}
Where $\mathrm{I}-\text{modified}$ and $\mathrm{II}-\text{modified}$ are defined as,
\begin{align*}
  \mathrm{I}-\text{modified} &=   \mathbb{E}\Big[\sum_{t=1}^T \sum_{j=1}^2 \left( f(\mathcal{A}^{(j)}_t, \pi^{(j)}_{s_{t,i_\star},i_\star}) - \mathbf{1}( t \in \mathbb{T}_{i_\star}^c \text{ or }  j = 2) \frac{U_{i_\star}(s_{t,i_\star}, \delta)}{s_{t,i_\star}}\right) -  \\
  &\qquad~\left( f(\mathcal{A}^{(j)}_t, \pi^{(j)}_t) - \frac{U_{j_t }(s_{t,j_t}, \delta) }{s_{t, j_t}}\right) \Big] \\
  \mathrm{II}-\text{modified} &= \mathbb{E}\Big[\sum_{t=1}^T \sum_{j=1}^2  f(\mathcal{A}^{(j)}_t, \pi^*) - \\
  &\qquad~\left( f(\mathcal{A}^{(j)}_t, \pi^{(j)}_{s_{t,i_\star},i_\star}) -\mathbf{1}( t \in \mathbb{T}_{i_\star}^c \text{ or }  j = 2) \frac{U_i(s_{t,i_\star}, \delta)}{s_{t,i_\star}}\right)\Big]
\end{align*}

We provided a bound for term $\mathrm{I}$-modified at the beginning of Section~\ref{app:ommited_proofs_stochastic_corral}. In this section we concern ourselves with $\mathrm{II}-$modified. Notice its expectation can be written as:

\begin{equation*}
    \EE{\mathrm{II}-\text{modified}} = \EE{\mathrm{II}} + \EE{\sum_{t=1}^T \sum_{j=1}^2 \mathbf{1}( t \in \mathbb{T}_{i_\star}^c \text{ or }  j = 2) \frac{U_{i_\star}(s_{t,i_\star}, \delta)}{s_{t,i_\star}}}
\end{equation*}

Now the second part of this sum is easy to deal with as it can be incorporated into the bound of $\EE{\mathrm{II}}$ by slightly modifying the bound given by Equation~\ref{equation::bound_I_0} below and changing $2b_l -1 $ to $2b_l + 1$. The rest of the argument remains the same.

\subsection*{ Bounding $\EE{\mathrm{II}}$ when $\underline{p}_{i_\star}$ is fixed}

From this section onward we drop the subscript $i_\star$ whenever clear to simplify the notations. In this section we show an upper bound for Term $\mathrm{II}$ when there is a value $\underline{p}_{i_\star}\in(0,1)$ that lower bounds $p_1^i, \cdots, p_T^{i_\star}$ with probability $1$. We then use the restarting trick to extend the proof to the case when $\underline{p}_i$ is random in Theorem~\ref{theorem::path_dependent_regret} 
\begin{lemma}[Fixed $\underline{p}_{i_\star}$]\label{lemma::fixed_pmin}
Let $\underline{p}_{i_\star} \in (0,1)$ be such that $\frac{1}{\rho_{i_\star}} = \underline{p}_{i_\star} \leq p_1^{i_\star}, \cdots, p_T^{i_\star}$ with probability one, then, $\EE{\mathrm{II}} \leq 4 \rho_{i_\star}\, U_i(T/\rho_{i_\star}, \delta) \log T + \delta T$. 
\end{lemma}

\begin{proof}[Proof of Lemma~\ref{lemma::fixed_pmin}]
Since $\EE{\mathrm{II}} \leq \EE{ \1\{\mathcal{E}\}\mathrm{II}} + \delta T$, we focus on bounding $\EE{ \1\{\mathcal{E}\}\mathrm{II}}$.  since base $i$ is $(U,T,\delta)-$bounded, $\EE{   R^{(1)}_{i_\star}(\mathbb{T}_i) \1(\mathcal{E}) } \leq \EE{ U_{i_\star}(\delta, n_T^{i_\star}) \1(\mathcal{E}) } $. We proceed to bound the regret corresponding to the remaining terms in $\mathrm{II}_0$:
\begin{align}
    \EE{\mathrm{II}_0\1(\mathcal{E})} &= \EE{\sum_{l=1}^{n_T^{i_\star} +1} \1\{\mathcal{E} \}(2b_l -1)\EE{r_{t_l, i_\star}^{(2)} | \mathcal{F}_{t_{l-1}}} }\notag \\
    &\leq \EE{ \sum_{l=1}^{n_T^{i_\star} +1} \1\{\mathcal{E} \}(2b_l -1) \frac{U_{i_\star}(l,\delta/2M)}{l} } \label{equation::bound_I_0}
\end{align}
The multiplier $2b_l-1$ arises because the policies proposed by the base algorithm during the rounds it is not selected by $\mathcal{M}$ satisfy $\pi_{t, i_\star}^{(1)} = \pi_{t, i_\star}^{(2)} = \pi_{t_l, i}^{(2)}$ for all $l\leq n_{i_\star}^T+1$ and $t = t_{l-1}+1, \cdots, t_l-1$. The factorization is a result of conditional independence between $\EE{r_{t_l, i_\star}^{(2)}| \mathcal{F}_{t_{l-1}}}$ and $\EE{b_l|\mathcal{F}_{t_{l-1}}}$ where $\mathcal{F}_{t_{l-1}}$ already includes algorithm $\widetilde{B}_{i_\star}$ update right after round $t_{l-1}$.
The inequality holds because $\widetilde{\mathcal{B}}_{i_\star}$ is $(U_{i_\star}, \frac{\delta}{2M}, \mathcal{T}^{(2)})-$smooth and therefore satisfies Equation \ref{equation::inst_regret_boundedness} on event $\mathcal{E}$. Recall that as a consequence of Equation~\ref{equation::decomposition_term_II} we have $$\EE{\mathrm{II}} \leq \EE{ R^{(1)}_{i_\star}(\mathbb{T}_i) \1(\mathcal{E})+ \mathrm{II}_0\1\{\mathcal{E}\} } + \delta T.  $$ The first term is bounded by $\EE{ U_{i_\star}(n_T^{i_\star}, \delta) \1(\mathcal{E}) } $ while the second term satisfies the bound in  \eqref{equation::bound_I_0}. Let $u_l = \frac{U_{i_\star}(l,\delta/2M)}{l}$. By Lemma~\ref{lem::helper}, $\sum_{l=1}^t u_l \geq U_{i_\star}(t, \delta/M)$ for all $t$, and so,
\begin{align}
\label{eq:Uu}
 \EE{\one{\mathcal{E}}U_{i_\star}(n_T^{i_\star}, \delta)}  
\leq \EE{  \sum_{l=1}^{n_T^{i_\star}+1} \one{\mathcal{E}} u_l} \;.
\end{align}
By \eqref{equation::bound_I_0} and \eqref{eq:Uu},
\begin{equation*}
    \EE{ R^{(1)}_{i_\star}(\mathbb{T}_{i_\star}) \1(\mathcal{E})+ \mathrm{II}_0\1\{\mathcal{E}\} } \leq \EE{ \sum_{l=1}^{n_T^{i_\star} +1} \1\{\mathcal{E} \}2b_l u_l } \;.
\end{equation*} 
Let $a_l = \EE{b_l}$ for all $l$. Consider a meta-algorithm that uses $\underline{p}_{i_\star}$ instead of $p_t^{i_\star}$. In this new process let $t_l'$ be the corresponding rounds when the base is selected, $\bar{n}_T^{i_\star}$ be the total number of rounds the base is selected, and $c_l = \EE{t_l' - t_{l-1}'}$. Since $\underline{p}_{i_\star} \leq p_t^{i_\star}$ for all $t$ it holds that $\sum_{l=1}^j a_l \leq \sum_{l=1}^j c_l$ for all $j$. If we use the same coin flips used to generate ${t_l}$ to generate ${t'_l}$, we observe that ${t'_l}\subset {t_l}$ and $\bar{n}_T^{i_\star} \le n_T^{i_\star}$. Let $f:\Real\rightarrow [0,1]$ be a decreasing function such that for integer $i_\star$, $f(i_\star)=u_{i_\star}$. Then $\sum_{l=1}^{n_T^{i_\star}+1} a_l u_l$ and $\sum_{l=1}^{\bar{n}_T^{i_\star}+1} c_l u_l$ are two estimates of integral $\int_0^T f(x)dx$. Given that ${t'_l}\subset {t_l}$ and $u_l$ is a decreasing sequence in $l$,
\[
\sum_{l=1}^{n_T^{i_\star}+1} \EE{t_l - t_{l-1}  }u_l \le \sum_{l=1}^{\bar{n}_T^{i_\star}+1} \EE{t_l' - t_{l-1}'} u_l\,,
\] %
and thus
\begin{equation*}
\EE{ R^{(1)}_{i_\star}(\mathbb{T}_{i_\star}) \1(\mathcal{E})+ \mathrm{II}_0\1\{\mathcal{E}\} } \leq  \mathbb{E} \sum_{l=1}^{\bar{n}_T^{i_\star}+1} 2\EE{t_l' - t_{l-1}'} u_l \;. 
\end{equation*}
We proceed to upper bound the right hand side of this inequality:
\begin{align*}
\EE{ \sum_{l=1}^{\bar{n}_T^{i_\star}+1} u_l \EE{t_l' - t_{l-1}'}} &\leq \EE{ \sum_{l=1}^{\bar{n}_T^{i_\star}+1} \frac{u_l}{\underline{p}_i} }\\
&\leq 2\rho_{i_\star} U_{i_\star}(T/\rho_{i_\star}, \delta) \log(T).
\end{align*}
The first inequality holds because $\EE{t_l'-t_{l-1}'}\leq \frac{1}{\underline{p}_{i_\star}}$ and the second inequality follows by concavity of $U_{i_\star}(t, \delta)$ as a function of $t$. The proof follows. 
\end{proof}

\subsection*{Proof of Theorem~\ref{theorem::path_dependent_regret}}

We use the restarting trick to extend Lemma~\ref{lemma::fixed_pmin} to the case when the lower bound $\underline{p}_{i_\star}$ is random (more specifically the algorithm (CORRAL) will maintain a lower bound that in the end will satisfy $\underline{p}_{i_\star} \approx \min_t p_t^{i_\star}$) in Theorem~\ref{theorem::path_dependent_regret}. We restate the theorem statement here for convenience.

\begin{theorem}[Theorem~\ref{theorem::path_dependent_regret} ]
\label{theorem::path_dependent_regret_appendix} 
\begin{equation*}
\EE{\mathrm{II}} \leq \mathcal{O}(\EE{ \rho_{i_\star}, U_{i_\star}(T/\rho_{i_\star}, \delta) \log T }+ \delta T(\log T +1)). 
\end{equation*}
Here, the expectation is over the random variable $\rho_{i_\star} = \max_{t}\frac{1}{p_t^{i_\star}}$. If $U(t, \delta) = t^\alpha c(\delta)$ for some $\alpha \in [1/2,1)$ then,
$\EE{\mathrm{II} }\leq 4\frac{2^{1-\alpha}}{2^{1-\alpha}-1}T^\alpha c(\delta)\EE{\rho_i^{1-\alpha}}+\delta T(\log T +1)$. 
\end{theorem}

{\textbf Restarting trick:} Initialize $\underline{p}_{i_\star} = \frac{1}{2M}$. If $p^{i_\star}_t < \underline{p}_{i_\star}$, set $\underline{p}_{i_\star} = \frac{p^{i_\star}_t}{2}$ and restart the base. 

\begin{proof}[Proof of Theorem~\ref{theorem::path_dependent_regret}]
The proof follows that of Theorem 15 in \citep{DBLP:conf/colt/AgarwalLNS17}. Let $\ell_1,\cdots, \ell_{d_i} < T$ be the rounds where Line 10 of the CORRAL is executed. Let $\ell_0 = 0$ and $\ell_{d_{i_\star}+1} = T$ for notational convenience. Let $e_l = [\ell_{l-1}+1, \cdots, \ell_l]$. Denote by $\underline{p}_{i_\star,\ell_l}$ the probability lower bound maintained by CORRAL during time-steps $t\in[\ell_{l-1}, \cdots , \ell_l]$ and $\rho_{i_\star,\ell_l} = 1/\underline{p}_{i_\star,\ell_l}$. In the proof of Lemma 13 in \citep{DBLP:conf/colt/AgarwalLNS17}, the authors prove $d_{i_\star} \leq \log(T)$ with probability one. Therefore,
\begin{small}
\begin{align*}
    \EE{\mathrm{II}} &= \sum_{l=1}^{\lceil\log(T)\rceil} \mathbb{P}( \underbrace{d_{i_\star} +1 \geq l}_{I(l)}  )\EE{ R_{i_\star}^{(1)}( e_l ) + R_{i_\star}^{(2)}(e_l)     | d_{i_\star} + 1 \geq l} \\
    &\leq \log T \sum_{l=1}^{\lceil\log(T)\rceil} \mathbb{P}(I(l)  )\EE{ 4 \rho_{i_\star,\ell_l} U_i(T/\rho_{i_\star,\ell_l}, \delta) |  I(l) } + \delta T (\log T + 1) \\
    &= \log T \EE{\sum_{l=1}^{b_i +1} 4 \rho_{i_\star,\ell_l} U_{i_\star}(T/\rho_{i_\star,\ell_l}, \delta)  } + \delta T (\log T + 1).
\end{align*}
\end{small}
The inequality is a consequence of Lemma \ref{lemma::fixed_pmin} applied to the restarted segment $[\ell_{l-1}, \cdots, \ell_l]$. This step is valid because by assumption $\frac{1}{\rho_{i_\star,\ell_l}} \leq \min_{t \in [\ell_{l-1}, \cdots, \ell_l]} p_t$.

If $U_{i_\star}(t, \delta) = t^\alpha c(\delta)$ for some function $c: \mathbb{R}\rightarrow \mathbb{R}^+$, then $\rho_{i_\star} U(T/\rho_{i_\star}, \delta) = \rho_{i_\star}^{1-\alpha} T^{\alpha} c(\delta)$. And therefore:
\begin{align*}
    \EE{\sum_{l=1}^{b_{i_\star} +1} \rho_{i_\star,\ell_l} U_{i_\star}(T/\rho_{i_\star,\ell_l}, \delta)  } &\leq T^\alpha g(\delta) \EE{\sum_{l=1}^{b_i +1} \rho_{i_\star,\ell_l}^{1-\alpha}  }\\
    &\leq \frac{2^{\bar{\alpha}}}{2^{\bar{\alpha}}-1} T^\alpha c(\delta)\EE{\rho_{i_\star}^{1-\alpha}}
\end{align*}
Where $\bar{\alpha} = 1-\alpha$. The last inequality follows from the same argument as in Theorem 15 in \citep{DBLP:conf/colt/AgarwalLNS17}.
\end{proof}
\subsection*{Proof of Theorem~\ref{thm:meta-algorithm}}

\begin{proof}
For the CORRAL meta-algorithm, 
\[
\EE{\mathrm{I}} \le \EE{\mathrm{I}_A} + \EE{\mathrm{I}_B} \le O\left( \frac{M \ln T}{\eta} + T \eta \right)- \frac{\EE{\rho}}{40 \eta \ln T} + 8 \sqrt{MT \log(\frac{4TM}{\delta}) }
\]
Using Theorem~\ref{theorem::path_dependent_regret} to control term $\mathrm{II}$, the total regret of CORRAL is: 
\begin{align*}
R(T) &\le \mathcal{O}\left( \frac{M \ln T}{\eta} + T \eta \right) -\EE{\frac{\rho}{40 \eta \ln T} - 2 \rho\, U(T/\rho, \delta) \log T} + \delta T + \\
&\quad~8 \sqrt{MT \log(\frac{4TM}{\delta}) }\\ 
&\le \mathcal{O}\left( \frac{M \ln T}{\eta} + T \eta \right) -\EE{\frac{\rho}{40 \eta \ln T} - 2 \rho^{1-\alpha} T^{\alpha} c( \delta) \log T} + \delta T + \\
&\quad~8 \sqrt{MT \log(\frac{4TM}{\delta}) }\\
&\le \widetilde{\mathcal{O}}\left( \sqrt{MT} + \frac{M }{\eta} + T\eta  + Tc(\delta)^{\frac{1}{\alpha}} \eta^{\frac{1-\alpha}{\alpha}} \right) +\delta T,
\end{align*} 
where the last step is by maximizing the function over $\rho$. Choose $\delta = 1/T$.  When both $\alpha$ and $c(\delta)$ are known, choose $\eta = \frac{M^{\alpha}}{c(\delta) T^{\alpha}}$. When only $\alpha$ is known, choose $\eta = \frac{M^{\alpha}}{ T^{\alpha}}$. 

For the EXP3.P meta-algorithm, if $p \le \frac{1}{2k}$:
\[
  \EE{I} \le \EE{\mathrm{I}_A} + \EE{\mathrm{I}_B} \le \widetilde{O} \left(MTp + \frac{\log(k\delta^{-1})}{p} + \sqrt{MT \log(\frac{4TM}{\delta})}\right)
\]
Using Lemma~\ref{lemma::fixed_pmin} to control term $\mathrm{II}$, we have the total regret of EXP3.P when $\delta = 1/T$:
\begin{align*}
 R(T) &= \widetilde{\mathcal{O}}(\sqrt{MT} + MTp +\frac{1}{p} +  \frac{1}{p} U_i(Tp, \delta) ) \;. \\
 &= \widetilde{\mathcal{O}}(\sqrt{MT} + MTp + T^{\alpha} p^{\alpha-1} c(\delta))
 \end{align*}
When both $\alpha$ and $c(\delta)$ are known, choose $p = T^{-\frac{ 1-\alpha}{2-\alpha}}M^{-\frac{1}{2-\alpha}}c(\delta)^{\frac{1}{2-\alpha }}$. When only $\alpha$ is known, choose $p = T^{-\frac{ 1-\alpha}{2-\alpha}}M^{-\frac{1}{2-\alpha}}$. 
We then have the following regret: 
\begin{small}
\begin{table*}
\begin{center}
\begin{tabular}{ |c|c|c| } 
 \hline
  EXP3.P & CORRAL \\ \hline
 $\widetilde{\mathcal{O}}\left( \sqrt{MT} +  MTp  + T^{\alpha} p^{\alpha-1} c(\delta) \right)$  & $\widetilde{\mathcal{O}}\left( \sqrt{MT}+ \frac{M }{\eta} + T\eta  + T\,c(\delta)^{\frac{1}{\alpha}} \eta^{\frac{1-\alpha}{\alpha}}\right) $ \\ \hline
 $\widetilde{\mathcal{O}}\left(\sqrt{MT} +  M^{\frac{1-\alpha}{2-\alpha}} T^{\frac{1}{2-\alpha}}  c(\delta)^{\frac{1}{2-\alpha}}\right)$& $\widetilde{\mathcal{O}}\left(\sqrt{MT} + M^{\alpha}T^{1-\alpha}+ M^{1-\alpha}T^{\alpha} c(\delta) \right)$ \\ \hline
 $\widetilde{\mathcal{O}}\left(\sqrt{MT} +  M^{\frac{1-\alpha}{2-\alpha}} T^{\frac{1}{2-\alpha}}  c(\delta)\right)$ & $\widetilde{\mathcal{O}}\left(\sqrt{MT} + M^{\alpha}T^{1-\alpha}+ M^{1-\alpha}T^{\alpha} c(\delta)^{\frac{1}{\alpha}} \right)$ \\ \hline
\end{tabular}
\end{center}
\vspace{-.8cm}
\caption[Regret guarantees for Stochastic CORRAL. ]{The top row shows the general regret guarantees. The middle row shows the regret guarantees when $\alpha$ and $c(\delta)$ are known. The bottom row shows the regret guarantees when $\alpha$ is known and $c(\delta)$ is unknown.}
\end{table*}
\end{small}
\end{proof}

\section{Ancillary Technical Results}
\label{app:lemmas}

\begin{lemma}
\label{lem::helper}
If $U(t, \delta) = t^\beta c(\delta)$, for $0 \leq \beta \leq 1$ then:
\begin{equation*}
  U(l, \delta) \leq  \sum_{t=1}^l \frac{U(t, \delta)}{t} \leq  \frac{1}{\beta} U(l, \delta)
\end{equation*}

\end{lemma}

\begin{proof}
The LHS follows immediately from observing $\frac{U(t, \delta)}{t}$ is decreasing as a function of $t$ and therefore $ \sum_{t=1}^l \frac{U(t, \delta)}{t} \geq l \frac{U(l, \delta)}{l} = U(l, \delta)$. The RHS is a consequence of bounding the sum by the integral $\int_{0}^l \frac{U(t, \delta)}{t}dt$, substituting the definition $U(t, \delta) = t^\beta c(\delta)$ and solving it. 
\end{proof}

\begin{lemma}
\label{lem::decreasing_average}
If $f(x)$ is a concave and doubly differentiable function on $x > 0$ and $f(0) \ge 0$ then $f(x)/x$ is decreasing on $x > 0$
\end{lemma}
\begin{proof}
In order to show that $f(x)/x$ is decreasing when $x>0$, we want to show that $\left( \frac{f(x)}{x} \right)' = \frac{x f'(x) - f(x)}{x^2} < 0$ when $x>0$. Since $0 f'(0) - f(0) \le 0$, we will show that $g(x) = x f'(x) - f(x)$ is a non-increasing function on $x>0$. We have $g'(x) = xf''(x) \le 0 $ when $x \ge 0$ because $f(x)$ is concave. Therefore $x f'(x) - f(x) \le 0 f'(0) - f(0) \le 0$ for all $x \ge 0$, which completes the proof. 

\end{proof}

\begin{lemma}
\label{lem::lower_bound_helper}
For any $\Delta\leq \frac{1}{4}: \mathrm{KL}(\frac{1}{2},\frac{1}{2}-\Delta)\leq 3\Delta^2$.
\end{lemma}
\begin{proof}
By definition $kl(p,q) = p\log(p/q) +(1-p)\log(\frac{1-p}{1-q})$, so
\begin{align*}
   \mathrm{KL}\left(\frac{1}{2},\frac{1}{2}-\Delta\right) &= \frac{1}{2}\left(\log(\frac{1}{1-2\Delta})+\log(\frac{1}{1+2\Delta})\right) \\
    &=\frac{1}{2}\log\left(\frac{1}{1-4\Delta^2}\right)=\frac{1}{2}\log\left(1+\frac{4\Delta^2}{1-4\Delta^2}\right)\\
    &\leq \frac{2\Delta^2}{1-4\Delta^2}\leq \frac{2\Delta^2}{\frac{3}{4}} \leq 3\Delta^2
\end{align*}
\end{proof}

\end{document}